%% file: main.tex
\documentclass[11pt]{article}
\usepackage{fullpage}
\usepackage[numbers]{natbib}
\usepackage{algorithm}
\usepackage[noend]{algorithmic}
\usepackage{amsmath,amsthm,amsfonts,amssymb}
\usepackage{amsmath}
\usepackage{hyperref}
\usepackage{color}
\usepackage{xcolor}
\usepackage{mathrsfs}
\usepackage{enumitem}
\usepackage{bm}
\usepackage{multirow}
\usepackage{booktabs}
\usepackage{makecell}
\usepackage{graphicx}
\usepackage{subfigure}
\usepackage{caption}
\usepackage{thmtools}
\usepackage{thm-restate}
\usepackage{setspace}
\usepackage{makecell}
\definecolor{light-gray}{gray}{0.85}
\usepackage{colortbl}
\usepackage{hhline}

\include{notations}
\usepackage{times}

\newtheorem{theorem}{Theorem}
\newtheorem{lemma}[theorem]{Lemma}
\newtheorem{corollary}[theorem]{Corollary}

\newtheorem{proposition}[theorem]{Proposition}
\newtheorem{conjecture}[theorem]{Conjecture}
\theoremstyle{definition}
\newtheorem{definition}[theorem]{Definition}

\newcommand{\setto}{\leftarrow}
\newcommand{\up}[1]{\overline{#1}}
\newcommand{\low}[1]{\underline{#1}}
\newcommand{\Reg}{{\rm Regret}}
\newcommand{\MG}{{\rm MG}}

\newcommand{\nash}{{\star}}

\title{\bf Near-Optimal Reinforcement Learning with Self-Play}

\author{%
  Yu Bai \\
  Salesforce Research\\
  \texttt{yu.bai@salesforce.com} \\
   \and
   Chi Jin \\
   Princeton University \\
   \texttt{chij@princeton.edu} \\
   \and
   Tiancheng Yu \\
  MIT \\
   \texttt{yutc@mit.edu} \\
}

\begin{document}
\maketitle

\input{abstract}
\input{intro}

\input{related}
\input{prelim}

\input{Q}

\input{V}
\input{lower}

\input{conclu}


\bibliographystyle{plainnat}
\bibliography{ref}

\clearpage

\appendix
\input{property}

\input{proof_Q}

\input{proof_V}
\input{proof_lower}

\input{bandit}

\end{document}

%% file: notations.tex
\newcommand{\paren}[1]{{\left( #1 \right)}}
\newcommand{\brac}[1]{{\left[ #1 \right]}}
\newcommand{\set}[1]{{\left\{ #1 \right\}}}

\newcommand{\defeq}{\mathrel{\mathop:}=}

\newcommand{\argmin}{\mathop{\rm argmin}}

\newcommand{\trans}{^{\top}}
\newcommand{\poly}{\mathrm{poly}}

\newcommand{\E}{\mathbb{E}}
\newcommand{\D}{\mathbb{D}}
\renewcommand{\P}{\mathbb{P}}

\newcommand{\la}{\langle}
\newcommand{\ra}{\rangle}

\newcommand{\cO}{\mathcal{O}}
\newcommand{\tlO}{\mathcal{\tilde{O}}}

\newcommand{\R}{\mathbb{R}}

\newcommand{\cF}{\mathcal{F}}

\newcommand{\cD}{\mathcal{D}}

\newcommand{\cS}{\mathcal{S}}
\newcommand{\cA}{\mathcal{A}}
\newcommand{\cB}{\mathcal{B}}

\newenvironment{proof-sketch}{\noindent{\bf Proof Sketch}
  \hspace*{1em}}{\qed\bigskip\\}
\newenvironment{proof-idea}{\noindent{\bf Proof Idea}
  \hspace*{1em}}{\qed\bigskip\\}
\newenvironment{proof-of-lemma}[1][{}]{\noindent{\bf Proof of Lemma {#1}}
  \hspace*{1em}}{\qed\bigskip\\}
\newenvironment{proof-of-proposition}[1][{}]{\noindent{\bf
    Proof of Proposition {#1}}
  \hspace*{1em}}{\qed\bigskip\\}
\newenvironment{proof-of-theorem}[1][{}]{\noindent{\bf Proof of Theorem {#1}}
  \hspace*{1em}}{\qed\bigskip\\}
\newenvironment{inner-proof}{\noindent{\bf Proof}\hspace{1em}}{
  $\bigtriangledown$\medskip\\}
\newenvironment{proof-attempt}{\noindent{\bf Proof Attempt}
  \hspace*{1em}}{\qed\bigskip\\}

%% file: abstract.tex

\begin{abstract}
This paper considers the problem of designing optimal algorithms for reinforcement learning in two-player zero-sum games. We focus on self-play algorithms which learn the optimal policy by playing against itself without any direct supervision.
In a tabular episodic Markov game with $S$ states, $A$ max-player actions and $B$ min-player actions, the best existing algorithm for finding an approximate Nash equilibrium requires $\tlO(S^2AB)$ steps of game playing, when only highlighting the dependency on $(S,A,B)$. In contrast, the best existing lower bound scales as $\Omega(S(A+B))$ and has a significant gap from the upper bound. This paper closes this gap for the first time: we propose an optimistic variant of the \emph{Nash Q-learning} algorithm with sample complexity $\tlO(SAB)$, and a new \emph{Nash V-learning} algorithm with sample complexity $\tlO(S(A+B))$. The latter result matches the information-theoretic lower bound in all problem-dependent parameters except for a polynomial factor of the length of each episode. In addition, we present a computational hardness result for learning the best responses against a fixed opponent in Markov games---a learning objective different from finding the Nash equilibrium.


\end{abstract}

%% file: intro.tex

\section{Introduction}
\label{sec:introduction}



A wide range of modern artificial intelligence challenges can be cast as a multi-agent reinforcement learning (multi-agent RL) problem, in which more than one agent performs sequential decision making in an interactive environment. Multi-agent RL has achieved significant recent success on traditionally challenging tasks, for example in the game of GO \citep{silver2016mastering,silver2017mastering}, Poker \citep{brown2019superhuman}, real-time strategy games \citep{vinyals2019grandmaster, openaidota}, decentralized controls or multiagent robotics systems \citep{brambilla2013swarm}, autonomous driving \citep{shalev2016safe}, as well as complex social scenarios such as hide-and-seek~\citep{baker2020emergent}.
In many scenarios, the learning agents even outperform the best human experts .

Despite the great empirical success, a major bottleneck for many existing RL algorithms is that they require a tremendous number of samples. For example, the biggest AlphaGo Zero model is trained on tens of millions of games and took more than a month to train~\citep{silver2017mastering}. While requiring such amount of samples may be acceptable in simulatable environments such as GO, it is not so in other sample-expensive real world settings such as robotics and autonomous driving. It is thus important for us to understand the \emph{sample complexity} in RL---how can we design algorithms that find a near optimal policy with a small number of samples, and what is the fundamental limit, i.e. the minimum number of samples required for any algorithm to find a good policy.


Theoretical understandings on the sample complexity for multi-agent RL are rather limited, especially when compared with single-agent settings. The standard model for a single-agent setting is an episodic Markov Decision Process (MDP) with $S$ states, and $A$ actions, and $H$ steps per episode. The best known algorithm can find an $\epsilon$ near-optimal policy in $\tilde{\Theta}({\rm poly}(H)SA/\epsilon^2)$ episodes, which matches the lower bound up to a single $H$ factor~\citep{azar2017minimax,dann2019policy}. In contrast, in multi-agent settings, the optimal sample complexity remains open even in the basic setting of two-player tabular Markov games \cite{shapley1953stochastic}, where the agents are required to find the solutions of the games---the Nash equilibria.
The best known algorithm, {\tt VI-ULCB}, finds an $\epsilon$-approximate Nash equilibrium in $\tlO({\rm poly}(H)S^2AB/\epsilon^2)$ episodes~\citep{bai2020provable}, where $B$ is the number of actions for the other player. The information theoretical lower bound is $\Omega({\rm poly}(H)S(A+B)/\epsilon^2)$. Specifically, the number of episodes required for the algorithm scales quadratically in both $S$ and $(A,B)$, and exhibits a gap from the linear dependency in the lower bound. 
This motivates the following question:
\begin{center}
  {\bf Can we design algorithms with near-optimal sample complexity for learning Markov games?}
\end{center}
In this paper, we present the first line of near-optimal algorithms for two-player Markov games that match the aforementioned lower bound up to a $\poly(H)$ factor. This closes the open problem for achieving the optimal sample complexity in all $(S,A,B)$ dependency. Our algorithm learns by playing against itself without requiring any direct supervision, and is thus a \emph{self-play} algorithm.








\subsection{Our contributions}
\begin{itemize}[leftmargin=*]
\item We propose an optimistic variant of \emph{Nash Q-learning} \cite{hu2003nash}, and prove that it achieves sample complexity $\tlO(H^5SAB/\epsilon^2)$ for finding an $\epsilon$-approximate Nash equilibrium in two-player Markov games (Section~\ref{sec:exp-SAB}). 
Our algorithm builds optimistic upper and lower estimates of $Q$-values, and computes the \emph{Coarse Correlated Equilibrium} (CCE) over this pair of $Q$ estimates as its execution policies for both players.

\item We design a new algorithm---\emph{Nash V-learning}---for finding approximate Nash equilibria, and show that it achieves sample complexity $\tlO(H^6S(A+B)/\epsilon^2)$ (Section~\ref{sec:exp-SA+B}). This improves upon Nash Q-learning in case $\min\set{A,B}>H$. It is also the first result that matches the minimax lower bound up to only a $\poly(H)$ factor. This algorithm builds optimistic upper and lower estimates of $V$-values, and features a novel combination of Follow-the-Regularized-Leader (FTRL) and standard Q-learning algorithm to determine its execution policies.

\item Apart from finding Nash equilibria, we prove that learning the best responses of fixed opponents in Markov games is as hard as learning parity with noise---a notoriously difficult problem that is believed to be computationally hard (Section~\ref{sec:lower}). As a corollary, this hardness result directly implies that achieving sublinear regret against \emph{adversarial opponents} in Markov games is also computationally hard, a result that first appeared in~\citep{radanovic2019learning}. This in turn rules out the possibility of designing efficient algorithms for finding Nash equilibria by running no-regret algorithms for each player separately.

\end{itemize}
In addition to above contributions, this paper also features a novel approach of extracting \emph{certified policies}---from the estimates produced by reinforcement learning algorithms such as Nash Q-learning and Nash V-learning---that are \emph{certified} to have similar performance as Nash equilibrium policies, even when facing against their best response (see Section \ref{sec:exp-SAB} for more details). We believe this technique 
could be of broader interest to the community.

%% file: related.tex
\subsection{Related Work}

\paragraph{Markov games}
Markov games (or stochastic games) are proposed in the early 1950s~\cite{shapley1953stochastic}. They are widely used to model multi-agent RL. Learning the Nash equilibria of Markov games has been studied in classical work~\citep{littman1994markov, littman2001friend, hu2003nash,hansen2013strategy}, where the transition matrix and reward are assumed to be known, or in the asymptotic setting where the number of data goes to infinity. These results do not directly apply to the non-asymptotic setting where the transition and reward are unknown and only a limited amount of data are available for estimating them.

A recent line of work tackles self-play algorithms for Markov games in the non-asymptotic setting with strong reachability assumptions.
Specifically, \citet{wei2017online} assumes no matter what strategy one agent sticks to, the other agent can always reach all states by playing a certain policy, and \citet{jia2019feature, sidford2019solving} assume access to simulators (or generative models) that enable the agent to directly sample transition and reward information for any state-action pair. These settings ensure that all states can be reached directly, so no sophisticated exploration is not required.

Very recently, \cite{bai2020provable,xie2020learning} study learning Markov games without these reachability assumptions, where exploration becomes essential. However, both results suffer from highly suboptimal sample complexity. We compare them with our results in Table~\ref{table:rate}. The results of \cite{xie2020learning} also applies to the linear function approximation setting.
We remark that the R-max algorithm~\citep{brafman2002r} does provide provable guarantees for learning Markov game, even in the setting of playing against the adversarial opponent, but using a definition of regret that is weaker than the standard regret. Their result does not imply any sample complexity result for finding Nash equilibrium policies.

\begin{table*}[!t]
    \renewcommand{\arraystretch}{1.6} 
    \centering
    \caption{\label{table:rate} Sample complexity (the required number of episodes) for algorithms to find $\epsilon$-approximate Nash equlibrium policies in zero-sum Markov games. }
    \begin{tabular}{|c|c|c|}
      \hline
      \textbf{Algorithm} & \textbf{Sample Complexity} & \textbf{Runtime} \\ \hline
      VI-ULCB \cite{bai2020provable} & $\tlO(H^4 S^2 AB/\epsilon^2)$  & PPAD-complete\\\hline
      VI-explore \cite{bai2020provable}  & $\tlO(H^5S^2AB/\epsilon^2)$ &  \multirow{4}{*}{Polynomial} \\ \cline{1-2}
      OMVI-SM  \cite{xie2020learning}  & $\tlO(H^4S^3A^3B^3/\epsilon^2)$ & \\ \hhline{|--~|}
      \cellcolor{light-gray} Optimistic Nash Q-learning   & $\tlO(H^5SAB/\epsilon^2)$ &\\ \hhline{|--~|}
      \cellcolor{light-gray} Optimistic Nash V-learning  & $\quad\tlO(H^6S(A+B)/\epsilon^2)\quad$& \\ \hline
      Lower Bound \cite{jin2018q,bai2020provable}  & $\Omega(H^3S(A+B)/\epsilon^2)$ & - \\ \hline
    \end{tabular}
\end{table*}

    \paragraph{Adversarial MDP}
    Another line of related work focuses on provably efficient algorithms for \emph{adversarial MDPs}. Most work in this line considers the setting with adversarial rewards \citep{zimin2013online, rosenberg2019online, jin2019learning}, because adversarial MDP with changing dynamics is computationally hard even under full-information feedback~\cite{yadkori2013online}. These results do not directly imply provable self-play algorithms in our setting, because the opponent in Markov games can affect both the reward and the transition.

    \paragraph{Single-agent RL}
There is a rich literature on reinforcement learning in MDPs \citep[see e.g.][]{jaksch2010near, osband2014generalization, azar2017minimax, dann2017unifying, strehl2006pac, jin2018q}. MDP is a special case of Markov games, where 
only a single agent interacts with a stochastic environment. For the tabular episodic setting with nonstationary dynamics and no simulators, the best sample complexity achieved by existing model-based and model-free algorithms are $\tilde{\mathcal{O}}(H^3SA/\epsilon^2)$ \cite{azar2017minimax} and $\tilde{\mathcal{O}}(H^4SA/\epsilon^2)$ \cite{jin2018q}, respectively, where $S$ is the number of states, $A$ is the number of actions, $H$ is the length of each episode. Both of them (nearly) match the lower bound $\Omega(H^3SA/\epsilon^2)$~\cite{jaksch2010near, osband2016lower, jin2018q}.

%% file: prelim.tex
\section{Preliminaries} \label{sec:prelim}

We consider zero-sum Markov Games (MG)~\citep{shapley1953stochastic, littman1994markov}, which are also known as stochastic games in the literature. Zero-sum Markov games are generalization of standard Markov Decision Processes (MDP) into the two-player setting, in which the \emph{max-player} seeks to maximize the total return and the \emph{min-player} seeks to minimize the total return. 

Formally, we denote a tabular episodic Markov game as $\MG(H, \cS, \cA, \cB, \P, r)$, where $H$ is the number of steps in each episode, $\cS$ is the set of states with $|\cS| \le S$, $(\cA, \cB)$ are the sets of actions of the max-player and the min-player respectively, $\P = \{\P_h\}_{h\in[H]}$ is a collection of transition matrices, so that $\P_h ( \cdot | s, a, b) $ gives the distribution over states if action pair $(a, b)$ is taken for state $s$ at step $h$, and $r = \{r_h\}_{h\in[H]}$ is a collection of reward functions, and $r_h \colon \cS \times \cA \times \cB \to [0,1]$ is the deterministic reward function at step $h$.  \footnote{We assume the rewards in $[0,1]$ for normalization. Our results directly generalize to randomized reward functions, since learning the transition is more difficult than learning the reward.}  

In each episode of this MG, we start with a \emph{fixed initial state} $s_1$. Then, at each step $h \in [H]$, both
players observe state $s_h \in \cS$, and the max-player picks action
$a_h \in \cA$ while the min-player picks action $b_h \in \cB$ simultaneously. Both players observe the actions of the opponents, receive reward
$r_h(s_h, a_h, b_h)$, and then the environment transitions to the next state
$s_{h+1}\sim\P_h(\cdot | s_h, a_h, b_h)$. The episode ends when
$s_{H+1}$ is reached.


\paragraph{Markov policy, value function}
A \emph{Markov} policy $\mu$ of the max-player is a collection of $H$ functions
$\{ \mu_h: \cS \rightarrow \Delta_{\cA} \}_{h\in [H]}$, which maps from a state to a distribution of actions. Here
$\Delta_{\cA}$ is the probability simplex over action set
$\cA$. Similarly, a policy $\nu$ of the min-player is a collection of
$H$ functions
$\{ \nu_h: \cS \rightarrow \Delta_{\cB} \}_{h\in [H]}$. We use
the notation $\mu_h(a|s)$ and $\nu_h(b|s)$ to present the probability
of taking action $a$ or $b$ for state $s$ at step $h$ under Markov policy $\mu$ or $\nu$
respectively.  

We use $V^{\mu, \nu}_{h} \colon \cS \to \mathbb{R}$ to denote the value
function at step $h$ under policy $\mu$ and $\nu$, so that
$V^{\mu, \nu}_{h}(s)$ gives the expected cumulative rewards
received under policy $\mu$ and $\nu$, starting from $s$ at step $h$:
\begin{equation} \label{eq:V_value}
\textstyle V^{\mu, \nu}_{h}(s) \defeq \E_{\mu, \nu}\left[\left.\sum_{h' =
        h}^H r_{h'}(s_{h'}, a_{h'}, b_{h'}) \right| s_h = s\right].
\end{equation}
We also define $Q^{\mu, \nu}_h:\cS \times \cA \times \cB \to \mathbb{R}$ to denote $Q$-value function at step $h$ so that
$Q^{\mu, \nu}_{h}(s, a, b)$ gives the cumulative rewards received under policy $\mu$ and $\nu$, starting from
$(s, a, b)$ at step $h$:
\begin{equation} \label{eq:Q_value}
\textstyle Q^{\mu, \nu}_{h}(s, a, b) \defeq  \E_{\mu,
    \nu}\left[\left.\sum_{h' = h}^H r_{h'}(s_{h'},  a_{h'}, b_{h'})
    \right| s_h = s, a_h = a, b_h = b\right].
\end{equation}
For simplicity, we use notation of operator $\P_h$ so that
$[\P_h V](s, a, b) \defeq \E_{s' \sim \P_h(\cdot|s, a,
  b)}V(s')$ for any value function $V$. We also use notation $[\D_\pi Q](s) \defeq \E_{(a, b) \sim \pi(\cdot, \cdot|s)} Q(s, a, b)$ for any action-value function $Q$. By definition of value functions, we have the Bellman
equation
\begin{align*}
  Q^{\mu, \nu}_{h}(s, a, b) =
  (r_h + \P_h V^{\mu, \nu}_{h+1})(s, a, b), \qquad   V^{\mu, \nu}_{h}(s)
  =  (\D_{\mu_h\times\nu_h} Q^{\mu, \nu}_h)(s) 
\end{align*}
for all $(s, a, b, h) \in \cS \times \cA \times \cB \times [H]$. We define $V^{\mu, \nu}_{H+1}(s) = 0$ for all $s \in \cS_{H+1}$.



\paragraph{Best response and Nash equilibrium}
For any Markov policy of the max-player $\mu$, there exists a \emph{best response} of the min-player, which is a Markov policy
$\nu^\dagger(\mu)$ satisfying $V_h^{\mu, \nu^\dagger(\mu)}(s) = \inf_{\nu} V_h^{\mu, \nu}(s)$ for
any $(s, h) \in \cS \times [H]$. 
Here the infimum is taken over all possible policies which are not necessarily Markovian (we will define later in this section). 
We define $V_h^{\mu, \dagger} \defeq V_h^{\mu, \nu^\dagger(\mu)}$. By symmetry, we
can also define $\mu^\dagger(\nu)$ and $V_h^{\dagger, \nu}$.  
It is further known (cf.~\citep{filar2012competitive}) that there exist Markov policies $\mu^\star$, $\nu^\star$
that are optimal against the best responses of the opponents, in the sense that
\begin{equation*}
 \textstyle V^{\mu^\star, \dagger}_h(s) = \sup_{\mu}
      V^{\mu, \dagger}_h(s), 
      \qquad  V^{\dagger, \nu^\star}_h(s) = \inf_{\nu}
      V^{\dagger, \nu}_h(s),
  \qquad \textrm{for all}~(s, h).
\end{equation*}
We call these optimal strategies $(\mu^\star,\nu^\star)$ the Nash equilibrium of the Markov game, which satisfies 
the following minimax equation: 
\footnote{The minimax theorem here is different from the one for matrix games, i.e. $\max_\phi\min_\psi \phi\trans A\psi = \min_\psi\max_\phi \phi\trans A\psi$ for any matrix $A$, since here $V^{\mu, \nu}_h(s)$ is in general not bilinear in $\mu, \nu$.}
\begin{equation*}
\textstyle \sup_{\mu} \inf_{\nu} V^{\mu, \nu}_h(s) = V^{\mu^\star, \nu^\star}_h(s) = \inf_{\nu} \sup_{\mu} V^{\mu, \nu}_h(s).
\end{equation*}
Intuitively, a Nash equilibrium gives a solution in which no player has anything to gain by changing only her own policy. 
We further abbreviate the values of Nash equilibrium $V_h^{\mu^\star, \nu^\star}$ and $Q_h^{\mu^\star, \nu^\star}$ as $V_h^{\nash}$ and $Q_h^{\nash}$.
We refer readers to Appendix \ref{app:bellman} for Bellman optimality equations for values of best responses or Nash equilibria.

\paragraph{General (non-Markovian) policy}
In certain situations, it is beneficial to consider general, history-dependent policies that are not necessarily Markovian. A \emph{(general) policy} $\mu$ of the max-player is a set of $H$ maps $\mu \defeq \big\{ \mu_h: \R \times (\cS \times \cA \times \cB \times \R)^{h-1}\times \cS \rightarrow \Delta_{\cA} \big\}_{h\in [H]}$, from a random number $z\in\R$ and a history of length $h$---say $(s_1, a_1, b_1, r_1, \cdots, s_h)$, to a distribution over actions in $\cA$. By symmetry, we can also define the (general) policy $\nu$ of the min-player, by replacing the action set $\cA$ in the definition by set $\cB$. The random number $z$ is sampled from some underlying distribution $\mathcal{D}$, but may be shared among all steps $h \in [H]$. 

For a pair of general policy $(\mu, \nu)$, we can still use the same definitions \eqref{eq:V_value} to define their value $V_1^{\mu, \nu}(s_1)$  at step $1$. We can also define the best response $\nu^\dagger(\mu)$ of a general policy $\mu$ as the minimizing policy so that $V_1^{\mu, \dagger}(s_1) \equiv V_1^{\mu, \nu^\dagger(\mu)}(s_1) = \inf_{\nu} V_h^{\mu, \nu}(s_1)$ at step 1. We remark that the best response of a general policy is not necessarily Markovian.

\paragraph{Learning Objective}

There are two possible learning objectives in the setting of Markov games. The first one is to find the best response for a fixed opponent. Without loss of generality, we consider the case where the learning agent is the max-player, and the min-player is the opponent.
\begin{definition} [$\epsilon$-approximate best response] \label{def:epsilon_best_response} For an opponent with an fixed unknown general policy $\nu$, a general policy $\hat{\mu}$ is the \textbf{$\epsilon$-approximate best response} if $V^{\dagger, \nu}_1(s_1) - V^{\hat{\mu}, \nu}_1(s_1) \le \epsilon$.
\end{definition}

The second goal is to find a Nash equilibrium of the Markov games. We measure the suboptimality of any pair of general policies $(\hat{\mu}, \hat{\nu})$ using the gap between their performance and the performance of the optimal strategy (i.e. Nash equilibrium) when playing against the best responses respectively:
\begin{equation*}
 \textstyle    \textstyle V^{\dagger, \hat{\nu}}_1(s_1) - V^{\hat{\mu}, \dagger}_1(s_1)  = \brac{V^{\dagger, \hat{\nu}}_1(s_1) - V^{\nash}_1(s_1)} +  \brac{V^{\nash}_1(s_1) -   V^{\hat{\mu},\dagger}_1(s_1)}
\end{equation*}

\begin{definition}[$\epsilon$-approximate Nash equilibrium] \label{def:epsilon_Nash} A pair of general policies $(\hat{\mu},\hat{\nu})$ is an \textbf{$\epsilon$-approximate Nash equilibrium}, if $V^{\dagger, \hat{\nu}}_1(s_1) - V^{\hat{\mu}, \dagger}_1(s_1) \le \epsilon$.
\end{definition}

Loosely speaking, Nash equilibria can be viewed as ``the best responses to the best responses''. In most applications, they are the ultimate solutions to the games. In Section \ref{sec:exp-SAB} and \ref{sec:exp-SA+B}, we present sharp guarantees for learning an approximate Nash equilibrium with near-optimal sample complexity. However, rather surprisingly, learning a best response in the worst case is more challenging than learning the Nash equilibrium. In Section \ref{sec:lower}, we present a computational hardness result for learning an approximate best response.







%% file: Q.tex

\section{Optimistic Nash Q-learning}
\label{sec:exp-SAB}

In this section, we present our first algorithm \emph{Optimistic Nash Q-learning} and its corresponding theoretical guarantees.

\begin{algorithm}[t]
   \caption{Optimistic Nash Q-learning}
   \label{algorithm:Nash_Q}
   \begin{algorithmic}[1]
      \STATE {\bfseries Initialize:} for any $(s, a, b, h)$,
      $\up{Q}_{h}(s,a, b)\setto H$, $\low{Q}_{h}(s,a, b)\setto 0$,
      $N_{h}(s,a, b)\setto 0$,\\\qquad\qquad~~$\pi_h(a, b|s) \leftarrow 1/(AB)$.
      \FOR{episode $k=1,\dots,K$}
      \STATE receive $s_1$.
      \FOR{step $h=1,\dots, H$}
      \STATE take action $(a_h, b_h) \sim  \pi_h(\cdot, \cdot| s_h)$.
      \STATE observe reward $r_h(s_h, a_h, b_h)$ and next state
      $s_{h+1}$.
      \STATE $t=N_{h}(s_h, a_h, b_h)\setto N_{h}(s_h, a_h, b_h) + 1$.
      \STATE $\up{Q}_h(s_h, a_h, b_h) \setto (1-\alpha_t)\up{Q}_h(s_h, a_h, b_h)+ \alpha_t(r_h(s_h, a_h, b_h)+\up{V}_{h+1}(s_{h+1})+\beta_t)$ \label{line:Q_up_update}
      \STATE $\low{Q}_h(s_h, a_h, b_h) \setto (1-\alpha_t)\low{Q}_h(s_h, a_h, b_h)+ \alpha_t(r_h(s_h, a_h, b_h)+\low{V}_{h+1}(s_{h+1})-\beta_t)$
      \label{line:Q_low_update}
      \STATE $\pi_h(\cdot, \cdot|s_h) \setto \textsc{CCE}(\up{Q}_h(s_h, \cdot, \cdot), \low{Q}_h(s_h, \cdot, \cdot))$
      \STATE $\up{V}_h(s_h) \leftarrow (\D_{\pi_h}\up{Q}_h)(s_h); \quad \low{V}_h(s_h) \leftarrow (\D_{\pi_h} \low{Q}_h)(s_h)$.
      \ENDFOR
      \ENDFOR
   \end{algorithmic}
\end{algorithm}

\paragraph{Algorithm part I: learning values}
Our algorithm \emph{Optimistic Nash Q-learning} (Algorithm \ref{algorithm:Nash_Q}) is an optimistic variant of Nash Q-learning \cite{hu2003nash}. For each step in each episode, it (a) takes actions according to the previously computed policy $\pi_h$, and observes the reward and next state, (b) performs incremental updates on Q-values, and (c) computes new greedy policies and updates $V$-values.
Part (a) is straightforward; we now focus on explaining part (b) and part (c).

In part (b), the incremental updates on Q-values (Line \ref{line:Q_up_update}, \ref{line:Q_low_update}) are almost the same as standard Q-learning \cite{watkins1989learning}, except here we maintain two separate Q-values---$\up{Q}_h$ and $\low{Q}_h$, as upper and lower confidence versions respectively. We add and subtract a bonus term $\beta_t$ in the corresponding updates, which depends on $t = N_{h}(s_h, a_h, b_h)$---the number of times $(s_h, a_h, b_h)$ has been visited at step $h$. We pick parameter $\alpha_t$ and $\beta_t$ as follows for 
some large constant $c$ , and log factors $\iota$:
\begin{equation}\label{eq:hyper_Nash_Q}
\textstyle \alpha_t= (H+1)/(H+t), \qquad \beta_t=c\sqrt{H^3\iota/t}
\end{equation}
In part (c), our greedy policies are computed using
a \emph{Coarse Correlated Equilibrium} (CCE) subroutine, which is first introduced by \cite{xie2020learning} to solve Markov games using value iteration algorithms. For any pair of matrices $\up{Q}, \low{Q} \in [0, H]^{A\times B}$, $\textsc{CCE}(\up{Q}, \low{Q})$ returns a distribution $\pi \in \Delta_{\cA \times \cB}$ such that
\begin{align}
\E_{(a, b) \sim \pi} \up{Q}(a, b) \ge& \max_{a^\star} \E_{(a, b) \sim \pi} \up{Q}(a^\star, b) \label{eq:CCE}\\
\E_{(a, b) \sim \pi} \low{Q}(a, b) \le& \min_{b^\star} \E_{(a, b) \sim \pi} \low{Q}(a, b^\star) \nonumber
\end{align}
It can be shown that a CCE always exists, and it can be computed by linear programming in polynomial time (see Appendix \ref{app:CCE} for more details).

Now we are ready to state an intermediate guarantee for optimistic Nash Q-learning. We assume the algorithm has played the game for $K$ episodes, and  we use $V^k, Q^k, N^k, \pi^k$ to denote values, visitation counts, and policies \emph{at the beginning} of the $k$-th episode in Algorithm \ref{algorithm:Nash_Q}.

\begin{lemma}  \label{lem:regret_Nash_Q}
  For any $p\in (0, 1]$, choose hyperparameters $\alpha_t, \beta_t$ as in \eqref{eq:hyper_Nash_Q} for a large absolute constant $c$ and $\iota = \log (SABT/p)$. Then, with probability at least $1-p$, Algorithm~\ref{algorithm:Nash_Q} has following guarantees
  \begin{itemize}
  \item $\up{V}_{h}^{k}(s) \ge V^{\star}_{h}(s) \ge \low{V}_{h}^{k} (s)$ for all $(s, h, k) \in \cS \times [H]\times [K] $.
  \item $(1/K)\cdot\sum_{k=1}^K( \up{V}_{1}^{k}-\low{V}_{1}^{k} ) ( s_1 )  \le \cO\paren{\sqrt{H^5SAB\iota/K}}$.
  \end{itemize}
\end{lemma}
Lemma \ref{lem:regret_Nash_Q} makes two statements. First, it claims that the $\up{V}_{h}^{k}(s)$ and $\low{V}_{h}^{k} (s)$ computed in Algorithm~\ref{algorithm:Nash_Q} are indeed upper and lower bounds of the value of the Nash equilibrium. Second, Lemma \ref{lem:regret_Nash_Q} claims that the averages of the upper bounds and the lower bounds are also very close to the value of Nash equilibrium $V_1^\star(s_1)$, where the gap decrease as $1/\sqrt{K}$. This implies that in order to learn the value $V_1^\star(s_1)$ up to $\epsilon$-accuracy, we only need $\cO(H^5 SAB\iota/\epsilon^2)$ episodes.

However, Lemma \ref{lem:regret_Nash_Q} has a significant drawback: it only guarantees the learning of the \emph{value} of Nash equilibrium. It does not imply 
that the policies $(\mu^k, \nu^k)$ used in Algorithm \ref{algorithm:Nash_Q} are close to the Nash equilibrium, 
which requires the policies to have a near-optimal performance even against their best responses. This is a major difference between Markov games and standard MDPs, and is the reason why standard techniques from the MDP literature does not apply here. To resolve this problem, we propose a novel way to extract a certified policy from the optimistic Nash Q-learning algorithm.

\begin{algorithm}[t]
   \caption{Certified Policy $\hat{\mu}$ of Nash Q-learning}
   \label{algorithm:Q-policy}
   \begin{algorithmic}[1]
   \STATE sample $k \leftarrow \text{Uniform}([K])$.
      \FOR{step $h=1, \dots,H$}
      \STATE observe $s_{h}$, and take action $a_{h} \sim \mu^{k}_h(\cdot|s_h)$. 
      \STATE observe $b_{h}$, and set $t \setto N^{k}_h(s_{h},a_{h},b_{h})$.
      \STATE sample $m \in [t]$ with $\P(m=i)=\alpha^i_t$.
      \STATE $k \setto k^m_{h}(s_{h},a_{h},b_{h})$
      \ENDFOR
   \end{algorithmic}
\end{algorithm}

\paragraph{Algorithm part II: certified policies}
We describe our procedure of executing the certified policy $\hat{\mu}$ of the max-player is described in Algorithm \ref{algorithm:Q-policy}. Above, $\mu^k_h, \nu^k_h$ denote the marginal distributions of $\pi^k_h$ produced in Algorithm \ref{algorithm:Nash_Q} over action set $\cA, \cB$ respectively. We also introduce the following quantities that directly induced by $\alpha_t$:
\begin{equation} \label{eq:alpha_parameter}
\textstyle \alpha_{t}^{0}:=\prod_{j=1}^t{\left( 1-\alpha _j \right)}, \,\, \alpha _{t}^{i}:=\alpha _i\prod_{j=i+1}^t{\left( 1-\alpha _j \right)}
\end{equation}
whose properties are listed in the following Lemma~\ref{lem:step_size}. Especially, $\sum_{i=1}^t \alpha_t^i = 1$, so $\{\alpha_t^i\}_{i=1}^t$ defines a distribution over $[t]$.  We use $k^m_{h}(s,a,b)$ to denote the index of the episode where $(s,a,b)$ is observed in step $h$ for the $m$-th time. The certified policy $\hat{\nu}$ of the min-player is easily defined by symmetry. We note that $\hat{\mu}, \hat{\nu}$ are clearly general policies, but they are no longer Markov policies.

The intuitive reason why such policy $\hat{\mu}$ defined in Algorithm \ref{algorithm:Q-policy} is certified by Nash Q-learning algorithm, is because the update equation in line \ref{line:Q_up_update} of Algorithm \ref{algorithm:Nash_Q} and equation \eqref{eq:alpha_parameter} gives relation:
\begin{equation*}
   \textstyle \up{Q}_{h}^{k}(s,a,b)=\alpha _{t}^{0}H+\sum_{i=1}^t{\alpha _{t}^{i}\left[ r_h(s,a,b)+\up{V}_{h+1}^{k_h^i(s,a,b)}(s_{h+1}^{k_h^i(s,a,b)}) +\beta_i \right]}
\end{equation*}
This certifies the good performance against the best responses if the max-player plays a mixture of policies $\{\mu_{h+1}^{k_h^i(s,a,b)}\}_{i=1}^t$ at step $h+1$ 
with mixing weights $\{\alpha_t^i\}_{i=1}^t$ (see Appendix \ref{app:Q_policy} for more details). 
A recursion of this argument leads to the certified policy $\hat{\mu}$---a nested mixture of policies.

We now present our main result for Nash Q-learning, using the certified policies $(\hat{\mu}, \hat{\nu})$.
\begin{theorem}[Sample Complexity of Nash Q-learning]
  \label{thm:sample_Nash_Q}
  For any $p\in (0,1]$, choose hyperparameters $\alpha_t, \beta_t$ as in \eqref{eq:hyper_Nash_Q} for large absolute constant $c$ and $\iota = \log (SABT/p)$.  Then, with probability at least $1-p$, if we run Nash Q-learning (Algorithm~\ref{algorithm:Nash_Q}) for $K$ episodes where
  \begin{equation*}
  \textstyle K \ge \Omega\left(H^5 SAB\iota/\epsilon^2\right),
  \end{equation*}
  the certified policies $(\hat{\mu}, \hat{\nu})$ (Algorithm \ref{algorithm:Q-policy}) will be $\epsilon$-approximate Nash, i.e. $V^{\dagger, \hat{\nu}}_1(s_1) - V^{\hat{\mu}, \dagger}_1(s_1) \le \epsilon$.
\end{theorem}
Theorem \ref{thm:sample_Nash_Q} asserts that if we run the optimistic Nash Q-learning algorithm for more than $\cO(H^5 SAB\iota/\epsilon^2)$ episodes, the certified policies $(\hat{\mu}, \hat{\nu})$ extracted using Algorithm \ref{algorithm:Q-policy} will be $\epsilon$-approximate Nash equilibrium (Definition \ref{def:epsilon_Nash}).

We make two remarks. First, the executions of the certified policies $\hat{\mu}, \hat{\nu}$ require the storage of $\{\mu^k_h\}$ and $\{\nu^k_h\}$ for all $k, h \in [H] \times [K]$. This makes the space complexity of our algorithm scales up linearly in the total number of episodes $K$. Second, Q-learning style algorithms (especially online updates) are crucial in our analysis for achieving sample complexity linear in $S$. They enjoy the property that every sample is only been used once, on the value function that is independent of this sample. 
In contrast, value iteration type algorithms do not enjoy such an independence property, which is why the best existing sample complexity scales as $S^2$ \citep{bai2020provable}. \footnote{Despite \cite{azar2017minimax} provides techniques to improve the sample complexity from $S^2$ to $S$ for value iteration in MDP, the same techniques can not be applied to Markov games due to the unique challenge that, in Markov games, we aim at finding policies that are good against their best responses.}

%% file: V.tex
\section{Optimistic Nash V-learning}
\label{sec:exp-SA+B}

In this section, we present our new algorithm \emph{Optimistic Nash V-learning} and its corresponding theoretical guarantees. This algorithm improves over Nash Q-learning in sample complexity from $\tlO(SAB)$ to $\tlO(S(A+B))$, when only highlighting the dependency on $S, A, B$. 

\begin{algorithm}[t]
   \caption{Optimistic Nash V-learning (the max-player version)}
   \label{algorithm:Nash_V}
   \begin{algorithmic}[1]
      \STATE {\bfseries Initialize:} for any $(s, a, b, h)$,
      $\up{V}_{h}(s)\setto H$, $\up{L}_h(s,a) \setto 0$, $N_{h}(s)\setto 0$, $\mu_h(a|s) \setto 1/A$.
      \FOR{episode $k=1,\dots,K$}
      \STATE receive $s_1$.
      \FOR{step $h=1,\dots, H$}
      \STATE take action $a_h \sim \mu_h(\cdot |s_h)$, observe the action $b_h$ from opponent.
      \STATE observe reward $r_h(s_h, a_h, b_h)$ and next state
      $s_{h+1}$.
      \STATE $t=N_{h}(s_h)\setto N_{h}(s_h) + 1$.
      \STATE $\up{V}_h(s_h) \setto \min\{H,(1-\alpha_t)\up{V}_h(s_h)+ \alpha_t(r_h(s_h, a_h, b_h)+\up{V}_{h+1}(s_{h+1})+\up{\beta}_t)\}$.
      \FOR{all $a \in \cA$}
      \STATE $\up{\ell}_h(s_h, a) \setto [H-r_h(s_h, a_h, b_h)-\up{V}_{h+1}(s_{h+1})] \mathbb{I}\{a_h=a\}/[\mu_h(a_h|s_{h}) +\up{\eta}_t]$.
      \STATE $\up{L}_h( s_h, a)  \setto ( 1-\alpha _t )\up{L}_h( s_h,a ) +\alpha_t \up{\ell}_h(s_h, a)$.
      \ENDFOR
      \STATE set $\mu_h(\cdot |s_h) \propto \exp[-(\up{\eta}_t/\alpha_t) \up{L}_h( s_h, \cdot)]$. \label{line:MWupdate}
      \ENDFOR
      \ENDFOR
   \end{algorithmic}
\end{algorithm}

\paragraph{Algorithm description}
Nash V-learning combines the idea of Follow-The-Regularized-Leader (FTRL) in the bandit literature with the Q-learning algorithm in reinforcement learning. This algorithm does not require extra information exchange between players other than standard game playing, thus can be ran separately by the two players. We describe the max-player version in Algorithm~\ref{algorithm:Nash_V}. See Algorithm \ref{algorithm:Nash_V_min} in Appendix \ref{sec:pf_V} for the min-player version, where $\low{V}_h$, $\low{L}_h$, $\nu_h$, $\low{\eta}_t$ and $\low{\beta}_t$ are defined symmetrically.



For each step in each episode, the algorithm (a) first takes action according to $\mu_h$, observes the action of the opponent, the reward, and the next state, (b) performs an incremental update on $\up{V}$, and (c) updates policy $\mu_h$. The first two parts are very similar to Nash Q-learning. In the third part, the agent first computes $\up{\ell}_h(s_h, \cdot)$ as the importance weighted estimator of the current loss. 
She then computes the weighted cumulative loss $\up{L}_h( s_h, \cdot)$. Finally, the policy $\mu_h$ is updated using FTRL principle:
\begin{equation*}
\textstyle \mu_h(\cdot|s_h) \leftarrow \argmin_{\mu \in \Delta_{\cA}} ~ \up{\eta}_t \la  \up{L}_h( s_h, \cdot), \mu\ra + \alpha_t \text{KL}(\mu \| \mu_0)
\end{equation*}
Here $\mu_0$ is the uniform distribution over all actions $\cA$. Solving above minimization problem gives the update equation as in Line \ref{line:MWupdate} in Algorithm \ref{algorithm:Nash_V}. In multi-arm bandit, FTRL can defend against adversarial losses, with regret independent of the number of the opponent's actions. This property turns out to be crucial for Nash V-learning to achieve sharper sample complexity than Nash Q-learning (see the analog of Lemma \ref{lem:regret_Nash_Q} in Lemma \ref{lem:regret_Nash_V}).

\begin{algorithm}[t]
   \caption{Certified Policy $\hat{\mu}$ of Nash V-learning}
   \label{algorithm:V-policy}
   \begin{algorithmic}[1]
   \STATE sample $k \leftarrow \text{Uniform}([K])$.
      \FOR{step $h=1, \dots,H$}
      \STATE observe $s_{h}$, and set $t \setto N^{k}_h(s_{h})$.
      \STATE sample $m \in [t]$ with $\P(m=i)=\alpha^i_t$.
      \STATE $k \setto k^m_{h}(s_{h})$.
      \STATE take action $a_{h} \sim \mu^{k}_h(\cdot|s_h)$.
      \ENDFOR
   \end{algorithmic}
\end{algorithm}

Similar to Nash Q-learning, we also propose a new algorithm (Algorithm \ref{algorithm:V-policy}) to extract a certified policy from the optimistic Nash V-learning algorithm. The certified policies are again non-Markovian. We choose all hyperparameters as follows, for some large constant $c$ , and log factors $\iota$.
 \begin{equation} \label{eq:hyper_Nash_V}
  \alpha_t=\frac{H+1}{H+t}, \quad\up{\eta}_t = \sqrt{\frac{\log A}{At}}, \quad\low{\eta}_t = \sqrt{\frac{\log B}{Bt}}, \quad\up{\beta}_t=c\sqrt{\frac{H^4A\iota}{t}}, \quad\low{\beta}_t=c\sqrt{\frac{H^4B\iota}{t}},
\end{equation}
We now present our main result on the sample complexity of Nash V-learning.

\begin{theorem}[Sample Complexity of Nash V-learning]
  \label{thm:sample_Nash_V}
  For any $p\in (0,1]$, choose hyperparameters as in \eqref{eq:hyper_Nash_V} for large absolute constant $c$ and $\iota = \log (SABT/p)$.  Then, with probability at least $1-p$, if we run Nash V-learning (Algorithm~\ref{algorithm:Nash_V} and \ref{algorithm:Nash_V_min}) for $K$ episodes with 
  \begin{equation*}
  \textstyle K \ge \Omega\left(H^6 S(A+B)\iota/\epsilon^2\right),
  \end{equation*}
  its induced policies $(\hat{\mu}, \hat{\nu})$ (Algorithm \ref{algorithm:V-policy}) will be $\epsilon$-approximate Nash, i.e. $V^{\dagger, \hat{\nu}}_1(s_1) - V^{\hat{\mu}, \dagger}_1(s_1) \le \epsilon$.
\end{theorem}

Theorem \ref{thm:sample_Nash_Q} claims that if we run the optimistic Nash V-learning for more than $\cO(H^6 S(A+B)\iota/\epsilon^2)$ episodes, the certified policies $(\hat{\mu}, \hat{\nu})$ extracted from Algorithm \ref{algorithm:V-policy} will be $\epsilon$-approximate Nash (Definition \ref{def:epsilon_Nash}). Nash V-learning is the first algorithm of which the sample complexity matches the information theoretical lower bound $\Omega(H^3S(A+B)/\epsilon^2)$ up to $\poly(H)$ factors and logarithmic terms.


%% file: lower.tex

\section{Hardness for Learning the Best Response}
\label{sec:lower}

In this section, we present a computational hardness result for computing the best response against an opponent with a fixed unknown policy. We further show that this implies the computational hardness result for achieving sublinear regret in Markov games when playing against adversarial opponents, which rules out a popular approach to design algorithms for finding Nash equilibria.

We first remark that if the opponent is restricted to only play Markov policies, then learning the best response is as easy as learning a optimal policy in the standard single-agent Markov decision process, where efficient algorithms are known to exist. Nevertheless, when the opponent can as well play any policy which may be non-Markovian, we show that finding the best response against those policies is computationally challenging.


We say an algorithm is a \emph{polynomial time algorithm for learning the best response} if for any policy of the opponent $\nu$, and for any $\epsilon > 0$, the algorithm finds the $\epsilon$-approximate best response of policy $\nu$ (Definition \ref{def:epsilon_best_response}) with probability at least $1/2$, in time polynomial in $S, H, A, B, \epsilon^{-1}$.

We can show the following hardness result for finding the best response in polynomial time.

\begin{theorem}[Hardness for learning the best response]
  \label{theorem:lower-bound-best-response}
There exists a Markov game with deterministic transitions and rewards defined for any horizon $H\ge 1$ with $S=2$, $A=2$, and $B=2$, such that if there exists a polynomial time algorithm for learning the best response for this Markov game, then there exists a polynomial time algorithm for learning parity with noise (see problem description in Appendix \ref{appendix:proof-lower}).
\end{theorem}

We remark that learning parity with noise is a notoriously difficult problem that has been used to design efficient cryptographic schemes.
It is conjectured by the community to be hard.
\begin{conjecture}[\cite{kearns1998efficient}]\label{conj:hardness}
There is no polynomial time algorithm for learning party with noise.
\end{conjecture}

Theorem \ref{theorem:lower-bound-best-response} with Conjecture \ref{conj:hardness} demonstrates the fundamental difficulty---if not strict impossibility---of designing a polynomial time for learning the best responses in Markov games. The intuitive reason for such computational hardness is that, while the underlying system has Markov transitions, the opponent can play policies that encode long-term correlations with non-Markovian nature, such as parity with noise, which makes it very challenging to find the best response. It is known that learning many other sequential models with long-term correlations (such as hidden Markov models or partially observable MDPs) is as hard as learning parity with noise \cite{mossel2005learning}.

\subsection{Hardness for Playing Against Adversarial Opponent}


Theorem \ref{theorem:lower-bound-best-response} directly implies the difficulty for achieving sublinear regret in Markov games when playing against adversarial opponents in Markov games. Our construction of hard instances in the proof of Theorem \ref{theorem:lower-bound-best-response} further allows the adversarial opponent to only play Markov policies in each episode.
Since playing against adversarial opponent is a different problem with independent interest, we present the full result here.



Without loss of generality, we still consider the setting where the algorithm can only control the max-player, while the min-player is an adversarial opponent. In the beginning of every episode $k$, both players pick their own policies $\mu^k$ and $\nu^k$, and execute them throughout the episode. The adversarial opponent can possibly pick her policy $\nu^k$ \emph{adaptive} to all the observations in the earlier episodes. 

We say an algorithm for the learner is a \emph{polynomial time no-regret algorithm} if there exists a $\delta >0$ such that for \emph{any} adversarial opponent, and any fixed $K > 0$, the algorithm outputs policies $\{\mu^k\}_{k=1}^K$ which satisfies the following, with probability at least $1/2$, in time polynomial in $S, H, A, B, K$.
\begin{equation}
  \label{equation:one-sided-regret}
  \Reg(K) = \sup_\mu \sum_{k=1}^K V^{\mu, \nu^k}_{1}
  (s_1) -  \sum_{k=1}^K V^{\mu^k, \nu^k}_{1} (s_1) \le {\rm poly}(S, H, A, B)K^{1-\delta}
\end{equation}


Theorem~\ref{theorem:lower-bound-best-response} directly implies the following hardness result for achieving no-regret against adversarial opponents, a result that first appeared in~\citep{radanovic2019learning}.
\begin{corollary}[Hardness for playing against adversarial opponent]
  \label{cor:lower-bound-coupled}
There exists a Markov game with deterministic transitions and rewards defined for any horizon $H\ge 1$ with $S=2$, $A=2$, and $B=2$, such that if there exists a polynomial time no-regret algorithm for this Markov game, then there exists a polynomial time algorithm for learning parity with noise (see problem description in Appendix \ref{appendix:proof-lower}).
The claim remains to hold even if we restrict the adversarial opponents in the Markov game to be non-adaptive, and to only play Markov policies in each episode.
\end{corollary}

Similar to Theorem \ref{theorem:lower-bound-best-response}, Corollary \ref{cor:lower-bound-coupled} combined with Conjecture \ref{conj:hardness} demonstrates the fundamental difficulty of designing a polynomial time no-regret algorithm against adversarial opponents for Markov games.


\paragraph{Implications on algorithm design for finding Nash Equilibria}
Corollary~\ref{cor:lower-bound-coupled} also rules out a natural approach for designing efficient algorithms for finding approximate Nash equilibrium through combining two no-regret algorithms. In fact, it is not hard to see that if the min-player also runs a non-regret algorithm, and obtain a regret bound symmetric to \eqref{equation:one-sided-regret}, then summing the two regret bounds shows the mixture policies $(\hat{\mu}, \hat{\nu})$---which assigns uniform mixing weights to policies $\{\mu^k\}_{k=1}^K$ and $\{\nu^k\}_{k=1}^K$ respectively---is an approximate Nash equilibrium. Corollary~\ref{cor:lower-bound-coupled} with Conjecture \ref{conj:hardness} claims that any algorithm designed using this approach is not a polynomial time algorithm.

%% file: conclu.tex
\section{Conclusion}
\label{sec:conclusion}
In this paper, we designed first line of near-optimal self-play algorithms for finding an approximate Nash equilibrium in two-player Markov games. The sample complexity of our algorithms matches the information theoretical lower bound up to only a polynomial factor in the length of each episode. Apart from finding Nash equilibria, we also prove the fundamental hardness in computation for finding the best responses of fixed opponents, as well as achieving sublinear regret against adversarial opponents, in Markov games.




%% file: property.tex

\section{Bellman Equations for Markov Games}
\label{app:bellman}

In this section, we present the Bellman equations for different types of values in Markov games.

\paragraph{Fixed policies.} For any pair of Markov policy $(\mu, \nu)$, by definition of their values in \eqref{eq:V_value} \eqref{eq:Q_value}, we have the following Bellman equations:
\begin{align*}
  Q^{\mu, \nu}_{h}(s, a, b) =  (r_h + \P_h V^{\mu, \nu}_{h+1})(s, a, b), \qquad   
  V^{\mu, \nu}_{h}(s)  =  (\D_{\mu_h\times\nu_h} Q^{\mu, \nu}_h)(s) 
\end{align*}
for all $(s, a, b, h) \in \cS \times \cA \times \cB \times [H]$, where $V^{\mu, \nu}_{H+1}(s) = 0$ for all $s \in \cS_{H+1}$.

\paragraph{Best responses.} For any Markov policy $\mu$ of the max-player, by definition, we have the following Bellman equations for values of its best response:
\begin{align*}
Q^{\mu, \dagger}_{h}(s, a, b) = (r_h + \P_h V^{\mu, \dagger}_{h+1})(s, a, b), \qquad
V^{\mu, \dagger}_{h}(s) = \inf_{\nu \in \Delta_{\cB}} (\D_{\mu_h \times \nu} Q^{\mu, \dagger}_h)(s),
\end{align*}
for all $(s, a, b, h) \in \cS \times \cA \times \cB \times [H]$, where $V^{\mu, \dagger}_{H+1}(s) = 0$ for all $s \in \cS_{H+1}$.

Similarly, for any Markov policy $\nu$ of the min-player, we also have the following symmetric version of Bellman equations for values of its best response:
\begin{align*}
Q^{\dagger, \nu}_{h}(s, a, b) =  (r_h + \P_h V^{\dagger, \nu}_{h+1})(s, a, b), \qquad 
V^{\dagger, \nu}_{h}(s) = \sup_{\mu \in \Delta_{\cA}} (\D_{\mu \times \nu_h} Q^{\dagger, \nu}_h)(s).
\end{align*}
for all $(s, a, b, h) \in \cS \times \cA \times \cB \times [H]$, where $V^{\dagger, \nu}_{H+1}(s) = 0$ for all $s \in \cS_{H+1}$.

\paragraph{Nash equilibria.} Finally, by definition of Nash equilibria in Markov games, we have the following Bellman optimality equations:
\begin{align*}
Q^{\nash}_{h}(s, a, b) = &  (r_h + \P_h V^{\nash}_{h+1})(s, a, b) \\
V^{\nash}_{h}(s) =&
  \sup_{\mu \in \Delta_{\cA}}\inf_{\nu \in \Delta_{\cB}} (\D_{\mu \times \nu} Q^{\nash}_h)(s)
  = \inf_{\nu \in \Delta_{\cB}}\sup_{\mu \in \Delta_{\cA}} (\D_{\mu \times \nu} Q^{\nash}_h)(s).
\end{align*}
for all $(s, a, b, h) \in \cS \times \cA \times \cB \times [H]$, where $V^{\nash}_{H+1}(s) = 0$ for all $s \in \cS_{H+1}$.


\section{Properties of Coarse Correlated Equilibrium}
\label{app:CCE}
Recall the definition for CCE in our main paper \eqref{eq:CCE}, we restate it here after rescaling. For any pair of matrix $P, Q \in [0, 1]^{n\times m}$, the subroutine $\textsc{CCE}(P, Q)$ returns a distribution $\pi \in \Delta_{n \times m}$ that satisfies:
\begin{align}
\E_{(a, b) \sim \pi} P(a, b) \ge& \max_{a^\star} \E_{(a, b) \sim \pi} P(a^\star, b)  \label{eq:constraints_CCE}\\
\E_{(a, b) \sim \pi} Q(a, b) \le& \min_{b^\star} \E_{(a, b) \sim \pi} Q(a, b^\star) \nonumber
\end{align}
We make three remarks on CCE. First, a CCE always exists since a Nash equilibrium for a general-sum game with payoff matrices $(P, Q)$ is also a CCE defined by $(P, Q)$, and a Nash equilibrium always exists. Second, a CCE can be efficiently computed, since above constraints \eqref{eq:constraints_CCE} for CCE can be rewritten as $n+m$ linear constraints on $\pi \in \Delta_{n \times m}$, which can be efficiently resolved by standard linear programming algorithm. Third, a CCE in general-sum games needs not to be a Nash equilibrium. However, a CCE in zero-sum games is guaranteed to be a Nash equalibrium.
\begin{proposition}\label{prop:CCE}
Let $\pi = \textsc{CCE}(Q, Q)$, and $(\mu, \nu)$ be the marginal distribution over both players' actions induced by $\pi$. Then $(\mu, \nu)$ is a Nash equilibrium for payoff matrix $Q$.
\end{proposition}

\begin{proof}[Proof of Proposition \ref{prop:CCE}]
Let $N^\star$ be the value of Nash equilibrium for $Q$. Since $\pi = \textsc{CCE}(Q, Q)$, by definition, we have:
\begin{align*}
\E_{(a, b) \sim \pi} Q(a, b) \ge& \max_{a^\star} \E_{(a, b) \sim \pi} Q(a^\star, b) = \max_{a^\star} \E_{b \sim \nu} Q(a^\star, b) \ge N^\star\\
\E_{(a, b) \sim \pi} Q(a, b) \le& \min_{b^\star} \E_{(a, b) \sim \pi} Q(a, b^\star) =\min_{b^\star} \E_{a \sim \mu} Q(a, b^\star) \le N^\star
\end{align*}
This gives:
\begin{equation*}
\max_{a^\star} \E_{b \sim \nu} Q(a^\star, b) = \min_{b^\star} \E_{a \sim \mu} Q(a, b^\star) = N^\star
\end{equation*}
which finishes the proof.
\end{proof}
Intuitively, a CCE procedure can be used in Nash Q-learning for finding an approximate Nash equilibrium, because the values of upper confidence and lower confidence---$\up{Q}$ and $\low{Q}$ will be eventually very close, so that the preconditions of Proposition \ref{prop:CCE} becomes approximately satisfied.



%% file: proof_Q.tex

\section{Proof for Nash Q-learning}
\label{app:Q}
In this section, we present proofs for results in Section~\ref{sec:exp-SAB}. 

We denote $V^k, Q^k, \pi^k$ for values and policies \emph{at the beginning} of the $k$-th episode. We also introduce the following short-hand notation
$[\widehat{\P}_{h}^{k}V](s, a, b) := V(s_{h+1}^{k})$.

We will use the following notations several times later: suppose $(s,a,b)$ was taken at the in episodes $k^1,k^2, \ldots $ at the $h$-th step. Since the definition of $k^i$ depends on the tuple $(s,a,b)$ and $h$, we will show the dependence explicitly by writing $k_h^i(s,a,b)$ when necessary and omit it when there is no confusion. We also define $N_{h}^{k}(s,a,b)$ to be the number of times $(s,a,b)$ has been taken \emph{at the beginning} of the $k$-th episode. Finally we denote $n_h^k=N_{h}^{k}\left(  s_{h}^{k},a_{h}^{k},b_{h}^{k} \right) $.

The following lemma is a simple consequence of the update rule in Algorithm~\ref{algorithm:Nash_Q}, which will be used several times later.

\begin{lemma}
   \label{lem:Nash_Q_V}
   Let $t=N_{h}^{k}\left( s,a,b \right) $ and suppose $(s,a,b)$ was previously taken at episodes $k^1,\ldots, k^t < k$ at the $h$-th step. The update rule in Algorithm \ref{algorithm:Nash_Q} is equivalent to the following equations.
   \begin{equation}
      \label{equ:upper-Q-decompose}
   \up{Q}_{h}^{k}(s,a,b)=\alpha _{t}^{0}H+\sum_{i=1}^t{\alpha _{t}^{i}\left[ r_h(s,a,b)+\up{V}_{h+1}^{k^i}(s_{h+1}^{k^i}) +\beta_i \right]}
   \end{equation}
   \begin{equation}
      \label{equ:lower-Q-decompose}
   \low{Q}_{h}^{k}(s,a,b)=\sum_{i=1}^t{\alpha _{t}^{i}\left[ r_h(s,a,b)+\low{V}_{h+1}^{k^i}(s_{h+1}^{k^i}) -\beta_i \right]}
\end{equation}
\end{lemma}

\subsection{Learning values}
We begin an auxiliary lemma. Some of the analysis in this section is adapted from \cite{jin2018q} which studies Q-learning under the single agent MDP setting.

\begin{lemma}(\citep[Lemma 4.1]{jin2018q})
\label{lem:step_size}
The following properties hold for $\alpha _{t}^{i}$:
\begin{enumerate}
\item \label{lem:lr_hoeffding}
$\frac{1}{\sqrt{t}} \le \sum_{i=1}^t \frac{\alpha^i_t}{\sqrt{i}} \le \frac{2}{\sqrt{t}}$ for every $t \ge 1$.
\item \label{lem:lr_property0}
$\max_{i\in[t]} \alpha^i_t \le \frac{2H}{t}$ and $\sum_{i=1}^t (\alpha^i_t)^2 \le \frac{2H}{t}$ for every $t \ge 1$.
\item \label{lem:lr_property}
$\sum_{t=i}^\infty \alpha^i_t = 1 + \frac{1}{H}$ for every $i \geq 1$.
\end{enumerate}
\end{lemma}

We also define $\tilde{\beta} _t:=2\sum_{i=1}^t{\alpha _{t}^{i}\beta_i} \le \cO(\sqrt{H^3\iota/t})$. Now we are ready to prove Lemma~\ref{lem:regret_Nash_Q}.

\begin{proof}[Proof of Lemma~\ref{lem:regret_Nash_Q}]
   We give the proof for one direction and the other direction is similar. For the proof of the first claim, let $t=N_{h}^{k}\left( s,a,b \right) $ and suppose $(s,a,b)$ was previously taken at episodes $k^1,\ldots, k^t < k$ at the $h$-th step. Let $\cF_i$ be the $\sigma$-algebra generated by all the random variables in until the $k^i$-th episode. Then $\{\alpha _{t}^{i}[ ( \widehat{\P}_{h}^{k^i}-\P_h ) V_{h+1}^{\star} ] \left( s,a,b \right)\}_{i=1}^t$ is a martingale differene sequence w.r.t. the filtration $\{\cF_i\}_{i=1}^t$.
    By Azuma-Hoeffding,
   $$
   \left| \sum_{i=1}^t{\alpha _{t}^{i}\left[ \left( \widehat{\P}_{h}^{k^i}-\P_h \right) V_{h+1}^{\star} \right] \left( s,a,b \right)} \right| \le 2H\sqrt{\sum_{i=1}^t{\left( \alpha _{t}^{i} \right) ^2\iota}} \le \tilde{\beta}_t
   $$

Here we prove a stronger version of the first claim by induction: for any $(s,a,b,h,k) \in \cS \times \cA \times \cB \times [H] \times [K]$,
$$ \up{Q}_{h}^{k}(s,a,b) \ge Q^{\star}_h(s,a,b) \ge \low{Q}_{h}^{k}(s,a,b), \,\,\,\, \up{V}_{h}^k(s) \ge V^{\star}_{h}(s) \ge \low{V}_{h}^k(s).
$$

Suppose the guarantee is true for $h+1$, then by the above concentration result, 
    $$
    (\up{Q}_{h}^{k}-Q^{\star}_h)(s,a,b) \ge \alpha _{t}^{0}H+\sum_{i=1}^t{\alpha _{t}^{i}\left( \up{V}_{h+1}^{k^i}-V_{h+1}^{\star} \right) \left( s_{h+1}^{k^i} \right)} \ge 0.
    $$
    
    Also,  
    \begin{align*}
      \up{V}_{h}^k(s) - V^{\star}_{h}(s)=& (\D_{\pi^k_{h}} \up{Q}_{h+1}^k)(s) - \max_{\mu \in \Delta_\cA}\min_{\nu \in \Delta_\cB}(\D_{\mu
      \times \nu} Q_{h+1}^{\star})(s)\\
\ge & \max_{\mu \in \Delta_\cA}
(\D_{\mu \times \nu^k_{h}} \up{Q}_{h}^k)(s) - \max_{\mu \in \Delta_\cA}   (\D_{\mu
\times \nu^k_{h}} Q_{h}^{\star})(s) \ge 0
    \end{align*}
where $\up{Q}_{h}^k (s,a,b) \ge Q_{h}^{\star}(s,a,b)$ has just been proved. The other direction is proved similarly.


   Now we continue with the proof of the second claim. Let $t=n_h^k$ and define $\delta _{h}^{k}:=\left( \up{V}_{h}^{k}-\low{V}_{h}^{k} \right) \left( s_{h}^{k} \right)$, then by definition
  \begin{align*}
 \delta _{h}^{k}=&\E_{\left( a,b \right) \sim \pi _{h}^{k}}\left( \up{Q}_{h}^{k}-\low{Q}_{h}^{k} \right) \left( s_{h}^{k},a,b \right) 
 =\left( \up{Q}_{h}^{k}-\low{Q}_{h}^{k} \right) \left( s_{h}^{k},a_{h}^{k},b_{h}^{k} \right) +\zeta_{h}^{k}
 \\
 &\overset{\left( i \right)}{=} \alpha _{t}^{0}H+\sum_{i=1}^t{\alpha _{t}^{i}\delta _{h+1}^{k_h^i(s_{h}^{k},a_{h}^{k},b_{h}^{k})}}+2\tilde{\beta} _t+\zeta_{h}^{k}
  \end{align*}  
 where $(i)$ is by taking the difference of equation~\eqref{equ:upper-Q-decompose} and equation~\eqref{equ:lower-Q-decompose} and 
 $$
 \zeta_{h}^{k}:=\E_{\left( a,b \right) \sim \pi _{h}^{k}}\left( \up{Q}_{h}^{k}-\low{Q}_{h}^{k} \right) \left( s_{h}^{k},a,b \right) -\left( \up{Q}_{h}^{k}-\low{Q}_{h}^{k} \right) \left( s_{h}^{k},a_{h}^{k},b_{h}^{k} \right)
 $$
 is a martingale difference sequence.

 Taking the summation w.r.t. $k$, we begin with the first two terms, 
   $$
   \sum_{k=1}^K{\alpha _{n_{h}^{k}}^{0}H}=\sum_{k=1}^K{H\mathbb{I}\left\{ n_{h}^{k}=0 \right\}}\le SABH
   $$
   $$
   \sum_{k=1}^K{\sum_{i=1}^{n_{h}^{k}}{\alpha _{n_{h}^{k}}^{i}\delta _{h+1}^{k_h^i\left( s_{h}^{k},a_{h}^{k},b_{h}^{k} \right)}}}\overset{\left( i \right)}{\le} \sum_{k'=1}^K{\delta _{h+1}^{k'}\sum_{i=n_{h}^{k'}+1}^{\infty}{\alpha _{i}^{n_{h}^{k'}}}}\overset{\left( ii \right)}{\le} \left( 1+\frac{1}{H} \right) \sum_{k=1}^K{\delta _{h+1}^{k}}.
   $$
  where $(i)$ is by changing the order of summation and $(ii)$ is by Lemma~\ref{lem:step_size}.

   Plugging them in, 
   \begin{align*}
      \sum_{k=1}^K{\delta _{h}^{k}}\le SABH+\left( 1+\frac{1}{H} \right) \sum_{k=1}^K{\delta _{h+1}^{k}}+\sum_{k=1}^K{\left( 2\tilde{\beta} _{n_{h}^{k}}+\zeta_{h}^{k} \right)}.
   \end{align*}

   Recursing this argument for $h \in [H]$ gives
   $$
   \sum_{k=1}^K{\delta _{1}^{k}}\le eSABH^2+2e\sum_{h=1}^H\sum_{k=1}^K{ \tilde{\beta} _{n_{h}^{k}}}+\sum_{h=1}^H\sum_{k=1}^K{(1+1/H)^{h-1}\zeta_{h}^{k}}
   $$

   By pigeonhole argument,
   $$
   \sum_{k=1}^K{\tilde{\beta} _{n_{h}^{k}}}\le \cO\left( 1 \right) \sum_{k=1}^K{\sqrt{\frac{H^3\iota}{n_{h}^{k}}}}=\cO\left( 1 \right) \sum_{s,a,b}{\sum_{n=1}^{N_{h}^{K}\left( s,a,b \right)}{\sqrt{\frac{H^3\iota}{n}}}}\le \cO\left( \sqrt{H^3SABK\iota} \right) =\cO\left( \sqrt{H^2SABT\iota} \right) 
$$

By Azuma-Hoeffding,
$$
\sum_{h=1}^H{\sum_{k=1}^K{(1+1/H)^{h-1}\zeta _{h}^{k}}}\le e\sqrt{2H^3K\iota}=eH\sqrt{2T\iota}
$$
with high probability. The proof is completed by putting everything together.
 
 \end{proof}

 \subsection{Certified policies} \label{app:Q_policy}

 Algorithm~\ref{algorithm:Nash_Q} only learns the value of game but itself cannot give a near optimal policy for each player. In this section, we analyze the certified policy based on the above exploration process (Algorithm~\ref{algorithm:Q-policy}) and prove the sample complexity guarantee. To this end, we need to first define a new group of policies $\hat{\mu}_{h}^{k}$ to facilitate the proof , and $\hat{\nu}_{h}^{k}$ are defined similarly. Notice $\hat{\mu}_{h}^{k}$ is related to $\hat{\mu}$ defined in Algorithm~\ref{algorithm:Q-policy} by $\hat{\mu}=\frac{1}{k}\sum_{i=1}^{k}{\hat{\mu}_{1}^{i}}$.

\begin{algorithm}[h]
   \caption{Policy $\hat{\mu}_{h}^{k}$}
   \label{algorithm:Q-sampling}
   \begin{algorithmic}[1]
      \STATE {\bfseries Initialize:} $k' \setto k$.
      \FOR{step $h'=h,h+1,\dots,H$}
      \STATE Observe $s_{h'}$.
      \STATE Sample $a_{h'} \sim \mu^{k'}_h(s_{h'})$. 
      \STATE Observe $b_{h'}$.
      \STATE $t \setto N^{k'}_h(s_{h'},a_{h'},b_{h'})$.
      \STATE Sample $i$ from $[t]$ with $\P(i)=\alpha^i_t$.
      \STATE $k' \setto k^i_{h'}(s_{h'},a_{h'},b_{h'})$
      \ENDFOR
   \end{algorithmic}
\end{algorithm}

We also define $\hat{\mu}_{h+1}^{k}[s,a,b]$ for $ h \le H-1$, which is na intermediate algorithm only involved in the analysis. The above two policies are related by $\hat{\mu}_{h+1}^{k}\left[ s,a,b \right] =\sum_{i=1}^t{\alpha _{t}^{i}}\hat{\mu}_{h+1}^{k}$ where $t=N_{h}^{k}\left( s,a,b \right) $. $\hat{\nu}_{h+1}^{k}[s,a,b]$ is defined similarly. 

\begin{algorithm}[h]
   \caption{Policy $\hat{\mu}_{h+1}^{k}[s,a,b]$}
   \label{algorithm:Q-sampling-sab}
   \begin{algorithmic}[1]
      \STATE $t \setto N^{k}_{h}(s,a,b)$.
      \STATE Sample $i$ from $[t]$ with $\P(i)=\alpha^i_t$.
      \STATE $k' \setto k^i_h(s,a,b)$
      \FOR{step $h'=h+1,\dots,H$}
      \STATE Observe $s_{h'}$.
      \STATE Sample $a_{h'} \sim \mu^{k'}_h(s_{h'})$. 
      \STATE Observe $b_{h'}$.
      \STATE $t \setto N^{k'}_h(s_{h'},a_{h'},b_{h'})$.
      \STATE Sample $i$ from $[t]$ with $\P(i)=\alpha^i_t$.
      \STATE $k' \setto k^i_{h'}(s_{h'},a_{h'},b_{h'})$
      \ENDFOR
   \end{algorithmic}
\end{algorithm}

Since the policies defined in Algorithm~\ref{algorithm:Q-sampling} and Algorithm~\ref{algorithm:Q-sampling-sab} are non-Markov, many notations for values of Markov policies are no longer valid here. To this end, we need to define the value and Q-value of general policies starting from step $h$, if the general policies starting from the $h$-th step do not depends the history before the $h$-th step. Notice the special case $h=1$ has already been covered in Section~\ref{sec:prelim}. For a pair of general policy $(\mu, \nu)$ which \emph{does not depend} on the hostory before the $h$-th step, we can still use the same definitions \eqref{eq:V_value} and \eqref{eq:Q_value} to define their value $V_h^{\mu, \nu}(s)$ and $Q_h^{\mu, \nu}(s,a,b)$ at step $h$. We can also define the best response $\nu^\dagger(\mu)$ of a general policy $\mu$ as the minimizing policy so that $V_h^{\mu, \dagger}(s) \equiv V_h^{\mu, \nu^\dagger(\mu)}(s) = \inf_{\nu} V_h^{\mu, \nu}(s)$ at step $h$. Similarly, we can define $Q_h^{\mu, \dagger}(s,a,b) \equiv Q_h^{\mu, \nu^\dagger(\mu)}(s,a,b) = \inf_{\nu} Q_h^{\mu, \nu}(s,a,b)$. As before, the best reponse of a general policy is not necessarily Markovian.

It should be clear from the definition of Algorithm~\ref{algorithm:Q-sampling} and Algorithm~\ref{algorithm:Q-sampling-sab} that $\hat{\mu}_{h}^{k}$, $\hat{\nu}_{h}^{k}$, $\hat{\mu}_{h+1}^{k}[s,a,b]$ and $\hat{\nu}_{h+1}^{k}[s,a,b]$ does not depend on the history before step $h$, therefore related value and Q-value functions are well defined for the corresponding steps. Now we can show the policies defined above are indeed certified.

\begin{lemma}
   \label{lem:Nash_q_ULCB_policy}  
   For any $p \in (0,1)$, with probability at least $1-p$, the following holds for any $(s,a,b,h,k) \in \cS \times \cA \times \cB \times [H] \times [K]$,
   $$
   \up{Q}_{h}^{k}(s,a,b) \ge Q_{h}^{\dag ,\hat{\nu}_{h+1}^k[s,a,b]}(s,a,b), \,\,\, \up{V}_{h}^{k}(s) \ge  V_{h}^{\dag ,\hat{\nu}_{h}^{k}}(s)
   $$
   $$
   \low{Q}_{h}^{k}(s,a,b) \le Q_{h}^{\hat{\mu}_{h+1}^k[s,a,b],\dag}(s,a,b), \,\,\, \low{V}_{h}^{k}(s) \le  V_{h}^{\hat{\mu}_{h}^{k},\dag}(s)
   $$
\end{lemma}

 \begin{proof}[Proof of Lemma~\ref{lem:Nash_q_ULCB_policy}]
    We first prove this for $h=H$. 
 \begin{align*}
\up{Q}_{H}^{k}(s,a,b)=&\alpha _{t}^{0}H+\sum_{i=1}^t{\alpha _{t}^{i}\left[ r_H(s,a,b)+\beta_i \right]}
\\
 \ge & r_H(s,a,b)
 =  Q_{H}^{\dag ,\hat{\nu}_{H+1}^{k}}\left( s,a,b \right)
 \end{align*}
 because $H$ is the last step and
 \begin{align*}
   \up{V}_{H}^{k}(s)=&(\D_{\pi^k_{H}} \up{Q}_{H}^k)(s) 
   \ge  \sup_{\mu \in \Delta_\cA}
   (\D_{\mu \times \nu^k_{H}} \up{Q}_{H}^k)(s)
   \\
    \ge & \sup_{\mu \in \Delta_\cA}
   (\D_{\mu \times \nu^k_{H}} r_{H})(s)
    = V_{H}^{\dag ,\nu_H^k}\left( s \right)
    =  V_{H}^{\dag ,\hat{\nu}_H^k}\left( s \right)
    \end{align*}
    because $\pi_H^k$ is CCE, and by definition $\hat{\nu}_{H}^{k}=\nu_{H}^{k}$.

   Now suppose the claim is true for $h+1$, consider the $h$ case. Consider a fixed tuple $(s,a,b)$ and let $t=N_{h}^{k}\left( s,a,b \right) $. Suppose $(s,a,b)$ was previously taken at episodes $k^1,\ldots, k^t < k$ at the $h$-th step. Let $\cF_i$ be the $\sigma$-algebra generated by all the random variables in until the $k^i$-th episode. Then $\{\alpha _{t}^{i}[ r_h(s,a,b)+V_{h+1}^{\dag ,\hat{\nu}_{h+1}^{k^i}}(s_{h+1}^{k^i}) +\beta_i ]\}_{i=1}^t$ is a martingale differene sequence w.r.t. the filtration $\{\cF_i\}_{i=1}^t$.
   By Azuma-Hoeffding and the definition of $b_i$, 
   $$
  \sum_{i=1}^t{\alpha _{t}^{i}\left[ r_h(s,a,b)+V_{h+1}^{\dag ,\hat{\nu}_{h+1}^{k^i}}(s_{h+1}^{k^i}) +\beta_i \right]} \ge \sum_{i=1}^t{\alpha _{t}^{i}Q_{h}^{\dag ,\hat{\nu}_{h+1}^{k^i}}(s,a,b)}
   $$
   with high probability. Combining this with the induction hypothesis,

   \begin{align*}
      \up{Q}_{h}^{k}(s,a,b)=&\alpha _{t}^{0}H+\sum_{i=1}^t{\alpha _{t}^{i}\left[ r_h(s,a,b)+\up{V}_{h+1}^{k^i}(s_{h+1}^{k^i}) +\beta_i \right]}
      \\
      \ge & \sum_{i=1}^t{\alpha _{t}^{i}\left[ r_h(s,a,b)+V_{h+1}^{\dag ,\hat{\nu}_{h+1}^{k^i}}(s_{h+1}^{k^i}) +\beta_i \right]}
      \ge \sum_{i=1}^t{\alpha _{t}^{i}Q_{h}^{\dag ,\hat{\nu}_{h+1}^{k^i}}(s,a,b)}
\\
\overset{\left( i \right)}{\ge} & \underset{\mu}{\max}\sum_{i=1}^t{\alpha _{t}^{i}Q_{h}^{\mu ,\hat{\nu}_{h+1}^{k^i}}(s,a,b)}
= Q_{h}^{\dag ,\hat{\nu}_{h+1}^{k}[s,a,b]}(s,a,b)
\end{align*}
where we have taken the maximum operator out of the summation in $(i)$,which does not increase the sum.

On the other hand,
\begin{align*}
   \up{V}_{h}^{k}(s)=&(\D_{\pi^k_{h}} \up{Q}_{h}^k)(s)
   \overset{\left( i \right)}{\ge} \sup_{\mu \in \Delta_\cA}
   (\D_{\mu \times \nu^k_{h}} \up{Q}_{h}^k)(s) 
   \\
   \overset{\left( ii \right)}{\ge} & \max_{a \in \cA} \E_{b \sim \nu^k_{h} }  Q_{h}^{\dag ,\hat{\nu}_{h+1}^{k}[s,a,b]}(s,a,b)
   = V_{h}^{\dag ,\hat{\nu}_{h}^{k}}(s)
\end{align*}
where $(i)$ is by the definition of CCE and $(ii)$ is the induction hypothesis. The other direction is proved by performing smilar arguments on $\low{Q}_{h}^{k}(s,a,b)$, $Q_{h}^{\hat{\mu}_{h+1}^k[s,a,b],\dag}(s,a,b)$, $\low{V}_{h}^{k}(s)$ and $ V_{h}^{\hat{\mu}_{h}^{k},\dag}(s)$.
 \end{proof}

Finally we give the theoretical guarantee of the policies defined above.

\begin{proof}[Proof of Theorem~\ref{thm:sample_Nash_Q}]
   By lemma~\ref{lem:Nash_q_ULCB_policy}, we have
$$
  \sum_{k=1}^K{\left(V_{1}^{\dag ,\hat{\nu}_{1}^k}-V_{1}^{\hat{\mu}_{1}^k,\dag } \right) \left( s_1 \right)}\le \sum_{k=1}^K{\left(\up{V}_{1}^{k}-\low{V}_{1}^{k} \right) \left( s_1 \right)} 
$$
and Lemma~\ref{lem:regret_Nash_Q} upper bounds this quantity by
$$
\sum_{k=1}^K{\left(V_{1}^{\dag ,\hat{\nu}_{1}^k}-V_{1}^{\hat{\mu}_{1}^k,\dag } \right) \left( s_1 \right)}\le \cO\left( \sqrt{H^4SABT\iota} \right) 
$$

By definition of the induced policy, with probability at least $1-p$, if we run Nash Q-learning (Algorithm~\ref{algorithm:Nash_Q}) for $K$ episodes with 
\begin{equation*}
K \ge \Omega\left(\frac{H^5 SAB\iota}{\epsilon^2}\right),
\end{equation*}
its induced policies $(\hat{\mu}, \hat{\nu})$ (Algorithm \ref{algorithm:Q-policy}) will be $\epsilon$-optimal in the sense $V^{\dagger, \hat{\nu}}_1(s_1) - V^{\hat{\mu}, \dagger}_1(s_1) \le \epsilon$.
\end{proof}

%% file: proof_V.tex

\section{Proof for Nash V-learning}
\label{sec:pf_V}


In this section, we present proofs of the results in Section~\ref{sec:exp-SA+B}.
We denote $V^k, \mu^k, \nu^k$ for values and policies \emph{at the beginning} of the $k$-th episode. We also introduce the following short-hand notation
$[\widehat{\P}_{h}^{k}V](s, a, b) := V(s_{h+1}^{k})$.

We will use the following notations several times later: suppose the state $s$ was visited at episodes $k^1,k^2, \ldots $ at the $h$-th step. Since the definition of $k^i$ depends on the state $s$ , we will show the dependence explicitly by writing $k_h^i(s)$ when necessary and omit it when there is no confusion. We also define $N_{h}^{k}(s)$ to be the number of times the state $s$ has been visited \emph{at the beginning} of the $k$-th episode. Finally we denote $n_h^k=N_{h}^{k}\left(  s_{h}^{k} \right) $. Notice the definitions here are different from that in Appendix~\ref{app:Q}.

The following lemma is a simple consequence of the update rule in Algorithm~\ref{algorithm:Nash_V}, which will be used several times later.

\begin{lemma}
   \label{lem:Nash_V_V}
   Let $t=N_{h}^{k}\left( s \right) $ and suppose $s$ was previously visited at episodes $k^1,\ldots, k^t < k$ at the $h$-th step. The update rule in Algorithm \ref{algorithm:Nash_V} is equivalent to the following equations.
   \begin{equation}
      \label{equ:upper-V-decompose}
      \up{V}_{h}^{k}(s)=\alpha _{t}^{0}H+\sum_{i=1}^t{\alpha _{t}^{i}\left[ r_h(s,a_h^{k^i},b_h^{k^i})+\up{V}_{h+1}^{k^i}(s_{h+1}^{k^i}) +\up{\beta}_i \right]}
   \end{equation}
   \begin{equation}
      \label{equ:lower-V-decompose}
      \low{V}_{h}^{k}(s)=\sum_{i=1}^t{\alpha _{t}^{i}\left[ r_h(s,a_h^{k_h^i},b_h^{k_h^i})+\low{V}_{h+1}^{k_h^i}(s_{h+1}^{k_h^i}) -\low{\beta}_i \right]}
\end{equation}
\end{lemma}

\subsection{Missing algorithm details}

We first give Algorithm~\ref{algorithm:Nash_V_min}: the min-player counterpart of Algorithm~\ref{algorithm:Nash_V}. Almost everything is symmetric except the definition of loss function to keep it non-negative.

\begin{algorithm}[t]
      \caption{Optimistic Nash V-learning (the min-player version)}
      \label{algorithm:Nash_V_min}
      \begin{algorithmic}[1]
         \STATE {\bfseries Initialize:} for any $(s, a, b, h)$,
         $\low{V}_{h}(s)\setto 0$, $\low{L}_h(s,b) \setto 0$, $N_{h}(s)\setto 0$, $\nu_h(b|s) \setto 1/B$.
         \FOR{episode $k=1,\dots,K$}
         \STATE receive $s_1$.
         \FOR{step $h=1,\dots, H$}
         \STATE take action $b_h \sim \nu_h(\cdot| s_h)$, observe the action $a_h$ from opponent
         \STATE observe reward $r_h(s_h, a_h, b_h)$ and next state $s_{h+1}$.
         \STATE $t=N_{h}(s_h)\setto N_{h}(s_h) + 1$.
         \STATE $\low{V}_h(s_h) \setto \max\{0,(1-\alpha_t)\low{V}_h(s_h)+ \alpha_t(r_h(s_h, a_h, b_h)+\low{V}_{h+1}(s_{h+1})-\low{\beta}_t)\}$
         \FOR{all $b \in \cB$}
         \STATE $\low{\ell}_h(s_h, b) \setto [r_h(s_h, a_h, b_h)+\low{V}_{h+1}(s_{h+1})] \mathbb{I}\{b_h=b\}/[\nu_h(b_h|s_{h}) +\low{\eta}_t]$.
         \STATE $\low{L}_h( s_h, b)  \setto ( 1-\alpha _t )\low{L}_h( s_h,b ) +\alpha_t \low{\ell}_h(s_h, b)$.
         \ENDFOR
         \STATE set $\nu_h(\cdot |s_h) \propto \exp[-(\low{\eta}_t/\alpha_t) \low{L}_h( s_h, \cdot)]$.
         \ENDFOR
         \ENDFOR
      \end{algorithmic}
   \end{algorithm}
\subsection{Learning values}

As usual, we begin with learning the value $V^{\star}$ of the Markov game. We begin with an auxiliary lemma, which justifies our choice of confidence bound.

\begin{lemma}
   \label{lem:bandit-regret}
   Let $t=N_{h}^{k}\left( s \right) $ and suppose state $s$ was previously taken at episodes $k^1,\ldots, k^t < k$ at the $h$-th step. Choosing $\up{\eta}_t=\sqrt{\frac{\log A}{At}}$ and $\low{\eta} _t=\sqrt{\frac{\log B}{Bt}}$, with probability $1-p$, for any $(s,h,t) \in \cS \times [H] \times [K]$, there exist a constant $c$ s.t.
$$
\underset{\mu}{\max}\sum_{i=1}^t{\alpha _{t}^{i}\D_{\mu \times \nu _{h}^{k^i}} \left( r_h +\P_h\up{V}_{h+1}^{k^i} \right)\left( s \right) }-\sum_{i=1}^t{\alpha _{t}^{i}\left[ r_h\left( s,a_{h}^{k^i},b_{h}^{k^i} \right) +\up{V}_{h+1}^{k^i}\left( s_{h+1}^{k^i} \right) \right]}\le c\sqrt{2H^4A\iota /t}
$$
$$
\sum_{i=1}^t{\alpha _{t}^{i}\left[ r_h\left( s,a_{h}^{k^i},b_{h}^{k^i} \right) +\low{V}_{h+1}^{k^i}\left( s_{h+1}^{k^i} \right) \right]}-\underset{\nu}{\min}\sum_{i=1}^t{\alpha _{t}^{i}\D_{\mu _{h}^{k^i} \times \nu} \left( r_h +\P_h\up{V}_{h+1}^{k^i} \right)\left( s \right)} \le c\sqrt{2H^4B\iota /t}
$$
\end{lemma}

\begin{proof}[Proof of Lemma~\ref{lem:bandit-regret}]
    We prove the first inequality. The proof for the second inequality is similar. We consider thoughout the proof a fixed $(s,h,t) \in \cS \times [H] \times [K]$. Define $\cF_i$ as the $\sigma$-algebra generated by all the random variables before the $k_h^i$-th episode. Then $\{r_h( s,a_{h}^{k^i},b_{h}^{k^i} ) +\up{V}_{h+1}^{k^i}( s_{h+1}^{k^i} )\}_{i=1}^{t}$ is a martingale sequence w.r.t. the filtration $\{\cF_i\}_{i=1}^{t}$. By Azuma-Hoeffding,
$$
\sum_{i=1}^t{\alpha _{t}^{i}\D_{\mu _{h}^{k^i} \times \nu _{h}^{k^i}} \left( r_h +\P_h\up{V}_{h+1}^{k^i} \right)\left( s \right) }-\sum_{i=1}^t{\alpha _{t}^{i}\left[ r_h\left( s,a_{h}^{k^i},b_{h}^{k^i} \right) +\up{V}_{h+1}^{k^i}\left( s_{h+1}^{k^i} \right) \right]}\le 2 \sqrt {H^3\iota /t}
$$
So we only need to bound
\begin{equation}\label{eq:regret_definition}
\underset{\mu}{\max}\sum_{i=1}^t{\alpha _{t}^{i}\D_{\mu \times \nu _{h}^{k^i}} \left( r_h +\P_h\up{V}_{h+1}^{k^i} \right)\left( s \right) } -\sum_{i=1}^t{\alpha _{t}^{i}\D_{\mu _{h}^{k^i} \times \nu _{h}^{k^i}} \left( r_h +\P_h\up{V}_{h+1}^{k^i} \right)\left( s \right) } := R^\star_t
\end{equation}
where $R^\star_t$ is the weighted regret in the first $t$ times of visiting state $s$, with respect to the optimal policy in hindsight, in the following adversarial bandit problem. The loss function is defined by
$$
l_i(a)=\E_{b \sim \nu _{h}^{k^i}\left( s \right)}\{H-h+1-r_h\left( s,a,b \right) -\P_h\up{V}_{h+1}^{k^i}\left( s,a,b \right)\}
$$
with weight $w_i = \alpha^{i}_t$. We note the weighted regret can be rewrite as $R_t^\star = \sum_{i=1}^t{w_i\left< \mu_h^\star - \mu_h^{k_i}, l_i \right>}$ where $\mu_h^\star$ is argmax for \eqref{eq:regret_definition}, and the loss function satisfies $l_i(a) \in [0, H]$

Therefore, Algorithm~\ref{algorithm:Nash_V} is essentially performing follow the regularized leader (FTRL) algorithm with changing step size for each state to solve this adversarial bandit problem. The policy we are using is $\mu _{h}^{k^i}\left( s,a \right)$ and the optimistic biased estimator
$$
\hat{l}_i(a)= \frac{H-h+1-r_h(s_h^{k^i}, a_h^{k^i}, b_h^{k^i})-\up{V}^{k^i}_{h+1}(s_{h+1}^{k^i})}{\mu _{h}^{k^i}\left( s,a \right) +\up{\eta} _i}\cdot\mathbb{I}\left\{ a_h^{k^i}=a \right\}
$$
is used to handle the bandit feedback.

A more detailed discussion on how to solve the weighted adversarial bandit problem is included in Appendix~\ref{sec:bandit}. Note that $w_i=\alpha_t^i$ is monotonic inscreasing, i.e. $\max _{i\le t}w_i=w_t$. By Lemma \ref{lem:adv-bandit}, we have
 \begin{align*}
   R^\star_t &\le 2H\alpha _{t}^{t}\sqrt{At\iota}+\frac{3H\sqrt{A\iota}}{2}\sum_{i=1}^t{\frac{\alpha _{t}^{i}}{\sqrt{i}}}+\frac{1}{2}H\alpha _{t}^{t}\iota +H\sqrt{2\iota \sum_{i=1}^t{\left( \alpha _{t}^{i} \right) ^2}}
\\
&\le 4H^2\sqrt{A\iota /t}+3H\sqrt{A\iota /t}+H^2\iota /t+\sqrt{4H^3\iota /t}
\\
&\le 10H^2\sqrt{A\iota /t}
 \end{align*}
 with probability $1-p/(SHK)$. Finally by a union bound over all $(s,h,t) \in \cS \times [H] \times [K]$, we finish the proof.
\end{proof}

We now prove the following Lemma~\ref{lem:regret_Nash_V}, which is an analoge of Lemma~\ref{lem:regret_Nash_Q} in Nash Q-learning.
\begin{lemma}  \label{lem:regret_Nash_V}
  For any $p\in (0, 1]$, choose hyperparameters as in \eqref{eq:hyper_Nash_V} for large absolute constant $c$ and $\iota = \log (SABT/p)$. Then, with probability at least $1-p$, Algorithm~\ref{algorithm:Nash_V} and \ref{algorithm:Nash_V_min} will jointly provide the following guarantees
  \begin{itemize}
  \item $\up{V}_{h}^{k}(s) \ge V^{\star}_{h}(s) \ge \low{V}_{h}^{k} (s)$ for all $(s, h, k) \in \cS \times [K] \times [H]$.
  \item $(1/K)\cdot \sum_{k=1}^K( \up{V}_{1}^{k}-\low{V}_{1}^{k} ) ( s_1 )  \le \cO\paren{\sqrt{H^6S(A+B)\iota/K}}$.
  \end{itemize}
\end{lemma}

\begin{proof}[Proof of Lemma~\ref{lem:regret_Nash_V}]

   We proof the first claim by backward induction. The claim is true for $h=H+1$. Asumme for any s, $\up{V}_{h+1}^{k}(s) \ge V_{h+1}^{\star}\left( s \right)$, $\low{V}_{h+1}^{k}(s) \le V_{h+1}^{\star}\left( s \right)$. For a fixed $(s,h) \in \cS \times [H]$ and episode $k \in [K]$, let $t=N_{h}^{k}\left( s \right) $ and suppose $s$ was previously visited at episodes $k^1,\ldots, k^t < k$ at the $h$-th step. By Bellman equation,
   \begin{align*}
      V_{h}^{\star}\left( s \right) =&\underset{\mu}{\max}\underset{\nu}{\min}\D_{\mu \times \nu} \left( r_h +\P_h V_{h+1}^{\star} \right) \left( s \right)
      \\
      =&\underset{\mu}{\max}\sum_{i=1}^t{\alpha _{t}^{i}}\underset{\nu}{\min}\D_{\mu \times \nu} \left( r_h +\P_h V_{h+1}^{\star} \right) \left( s \right)
      \\
      \le& \underset{\mu}{\max}\sum_{i=1}^t{\alpha _{t}^{i}\D_{\mu \times \nu _{h}^{k^i}} \left( r_h +\P_h V_{h+1}^{\star} \right) \left( s \right)}
      \\
      \le& \underset{\mu}{\max}\sum_{i=1}^t{\alpha _{t}^{i}\D_{\mu \times \nu _{h}^{k^i}} \left( r_h +\P_h\up{V}_{h+1}^{k^i} \right) \left( s \right)}
   \end{align*}

   Comparing with the decomposition of $\up{V}_{h}^{k}(s)$ in Equation~\eqref{equ:upper-V-decompose} and use Lemma~\ref{lem:bandit-regret}, we can see if $\up{\beta}_t=c\sqrt{AH^4\iota /t}$, then
   $\up{V}_{h}^{k}(s) \ge V_{h}^{\star}\left( s \right)$. Similar by taking $\low{\beta}_t=c\sqrt{BH^4\iota /t}$, we also have $\low{V}_{h}^{k}(s) \le V_{h}^{\star}\left( s \right)$.

   The second cliam is to bound $\delta _{h}^{k}:= \up{V}_{h}^{k}(s_h^k)-\low{V}_{h}^{k}(s_h^k) \ge 0$. Similar to what we have done in Nash Q-learning analysis, taking the difference of Equation~\eqref{equ:upper-V-decompose} and Equation~\eqref{equ:lower-V-decompose},
\begin{align*}
\delta _{h}^{k}=& \up{V}_{h}^{k}(s_h^k)-\low{V}_{h}^{k}(s_h^k)
\\
=&\alpha _{n_h^k}^{0}H+\sum_{i=1}^{n_h^k}{\alpha _{n_h^k}^{i}\left[\left( \up{V}_{h+1}^{k_h^i(s_{h}^{k})}-\low{V}_{h+1}^{k_h^i(s_{h}^{k})} \right)\left( s_{h+1}^{k_h^i(s_{h}^{k})} \right) +\up{\beta}_i+ \low{\beta}_i \right]}
\\
=&\alpha _{n_h^k}^{0}H+\sum_{i=1}^{n_h^k}{\alpha _{n_h^k}^{i}\delta _{h+1}^{k_h^i(s_{h}^{k})}}+\tilde{\beta} _{n_h^k}
\end{align*}
where
$$
\tilde{\beta} _j:=\sum_{i=1}^j{\alpha _{j}^{i}(\up{b}_i+ \low{b}_i}) \le c \sqrt{(A+B)H^4\iota/j}.
$$

Taking the summation w.r.t. $k$, we begin with the first two terms,
$$
\sum_{k=1}^K{\alpha _{n_{h}^{k}}^{0}H}=\sum_{k=1}^K{H\mathbb{I}\left\{ n_{h}^{k}=0 \right\}}\le SH
$$
$$
\sum_{k=1}^K{\sum_{i=1}^{n_{h}^{k}}{\alpha _{n_{h}^{k}}^{i}\delta _{h+1}^{k_h^i\left( s_{h}^{k} \right)}}}\overset{\left( i \right)}{\le} \sum_{k'=1}^K{\delta _{h+1}^{k'}\sum_{i=n_{h}^{k'}+1}^{\infty}{\alpha _{i}^{n_{h}^{k'}}}}\overset{\left( ii \right)}{\le} \left( 1+\frac{1}{H} \right) \sum_{k=1}^K{\delta _{h+1}^{k}}.
$$
where $(i)$ is by changing the order of summation and $(ii)$ is by Lemma~\ref{lem:step_size}. Putting them together,

\begin{align*}
    \sum_{k=1}^K{\delta _{h}^{k}}=&\sum_{k=1}^K{\alpha _{n_{h}^{k}}^{0}H}+\sum_{k=1}^K{\sum_{i=1}^{n_{h}^{k}}{\alpha _{n_{h}^{k}}^{i}\delta _{h+1}^{k_h^i\left( s_{h}^{k} \right)}}}+\sum_{k=1}^K{\tilde{\beta} _{n_{h}^{k}}}
\\
\le& HS+\left( 1+\frac{1}{H} \right) \sum_{k=1}^K{\delta _{h+1}^{k}}+\sum_{k=1}^K{\tilde{\beta} _{n_{h}^{k}}}
\end{align*}
Recursing this argument for $h \in [H]$ gives
   $$
   \sum_{k=1}^K{\delta _{1}^{k}}\le eSH^2+e\sum_{h=1}^H\sum_{k=1}^K{\tilde{\beta} _{n_{h}^{k}}}
   $$

   By pigeonhole argument,
   \begin{align*}
   \sum_{k=1}^K{\tilde{\beta} _{n_{h}^{k}}}&\le \cO\left( 1 \right) \sum_{k=1}^K{\sqrt{\frac{(A+B)H^4\iota}{n_{h}^{k}}}}
   =\cO\left( 1 \right) \sum_{s}{\sum_{n=1}^{n_{h}^{K}\left( s \right)}{\sqrt{\frac{(A+B)H^4\iota}{n}}}}\\
   &\le \cO\left( \sqrt{H^4S(A+B)K\iota} \right) =\cO\left( \sqrt{H^3S(A+B)T\iota} \right)
   \end{align*}

Expanding this formula repeatedly and apply pigeonhole argument we have
\begin{equation*}
\sum_{k=1}^K{[\up{V}_{h}^{k}-\low{V}_{h}^{k}](s_1)} \le \cO(\sqrt{H^5S(A+B)T\iota}).
\end{equation*}
which finishes the proof.
\end{proof}

\subsection{Certified policies}


As before, we construct a series of new policies  $\hat{\mu}_{h}^{k}$ in Algorithm~\ref{algorithm:V-sampling}. Notice $\hat{\mu}_{h}^{k}$ is related to $\hat{\mu}$ defined in Algorithm~\ref{algorithm:V-policy} by $\hat{\mu}=\frac{1}{k}\sum_{i=1}^{k}{\hat{\mu}_{1}^{i}}$. Also we need to consider value and Q-value functions of general policies which \emph{does not depend} on the hostory before the $h$-th step. See Appendix~\ref{app:Q_policy} for details. Again, we can show the policies defined above are indeed certified.

\begin{algorithm}[t]
   \caption{Policy $\hat{\mu}_{h}^{k}$}
   \label{algorithm:V-sampling}
   \begin{algorithmic}[1]
   \STATE sample $k \leftarrow \text{Uniform}([K])$.
      \FOR{step $h'=h,h+1,\dots,H$}
      \STATE observe $s_{h'}$, and set $t \setto N^{k}_{h'}(s_{h'})$.
      \STATE sample $m \in [t]$ with $\P(m=i)=\alpha^i_t$.
      \STATE $k \setto k^m_{h'}(s_{h'})$.
      \STATE take action $a_{h'} \sim \mu^{k}_{h'}(\cdot|s_{h'})$.
      \ENDFOR
   \end{algorithmic}
\end{algorithm}

\begin{lemma}
   \label{lem:v_ULCB_policy}
   For any $p \in (0,1)$, with probability at least $1-p$, the following holds for any $(s,a,b,h,k) \in \cS \times \cA \times \cB \times [H] \times [K]$,
   $$
 \up{V}_{h}^{k}(s) \ge  V_{h}^{\dag ,\hat{\nu}_{h}^{k}}(s), \,\,\,
 \low{V}_{h}^{k}(s) \le  V_{h}^{\hat{\mu}_{h}^{k},\dag}(s)
   $$
\end{lemma}

\begin{proof}[Proof of Lemma~\ref{lem:v_ULCB_policy}]
    We prove one side by induction and the other side is similar. The claim is trivially satisfied for $h=H+1$. Suppose it is ture for $h+1$, consider a fixed state $s$. Let $t=N_{h}^{k}\left( s \right) $ and suppose $s$ was previously visited at episodes $k^1,\ldots, k^t < k$ at the $h$-th step. Then using Lemma~\ref{lem:Nash_V_V},
\begin{align*}
   \up{V}_{h}^{k}(s)&=\alpha _{t}^{0}H+\sum_{i=1}^t{\alpha _{t}^{i}\left[ r_h(s,a_h^{k^i},b_h^{k_h^i})+\up{V}_{h+1}^{k^i}(s_{h+1}^{k^i}) +\up{\beta}_i \right]}
\\
&\overset{\left( i \right)}{\ge} \underset{\mu}{\max}\sum_{i=1}^t{\alpha _{t}^{i}\D_{\mu \times \nu _{h}^{k^i} }\left( r_h +\P_h\up{V}_{h+1}^{k^i} \right)\left( s \right)}
\\
&\overset{\left( ii \right)}{\ge} \underset{\mu}{\max}\sum_{i=1}^t{\alpha _{t}^{i}\D_{\mu \times \nu _{h}^{k^i} }\left( r_h +\P_h V_{h+1}^{\dag ,\hat{\nu}_{h+1}^{k^i}} \right)\left( s \right)}
\\
&=V_{h}^{\dag ,\hat{\nu}_{h}^{k}}(s)
\end{align*}
where $(i)$ is by using Lemma~\ref{lem:bandit-regret} and the definition of $\up{\beta}_i$, and $(ii)$ is by induction hypothesis.
\end{proof}

Equipped with the above lemmas, we are now ready to prove Theorem~\ref{thm:sample_Nash_V}.

\begin{proof}[Proof of Theorem~\ref{thm:sample_Nash_V}]
   By lemma~\ref{lem:v_ULCB_policy}, we have
$$
  \sum_{k=1}^K{\left(V_{1}^{\dag ,\hat{\nu}_{1}^k}-V_{1}^{\hat{\mu}_{1}^k,\dag } \right) \left( s_1 \right)}\le \sum_{k=1}^K{\left(\up{V}_{1}^{k}-\low{V}_{1}^{k} \right) \left( s_1 \right)}
$$
and Lemma~\ref{lem:regret_Nash_V} upper bounds this quantity by
$$
   \sum_{k=1}^K{\left(V_{1}^{\dag ,\hat{\nu}_{1}^k}-V_{1}^{\hat{\mu}_{1}^k,\dag } \right) \left( s_1 \right)}\le \cO\left( \sqrt{H^5S(A+B)T\iota} \right)
$$

By definition of the induced policy, with probability at least $1-p$, if we run Nash V-learning (Algorithm~\ref{algorithm:Nash_V}) for $K$ episodes with
\begin{equation*}
K \ge \Omega\left(\frac{H^6 S(A+B)\iota}{\epsilon^2}\right),
\end{equation*}
its induced policies $(\hat{\mu}, \hat{\nu})$ (Algorithm \ref{algorithm:V-policy}) will be $\epsilon$-optimal in the sense $V^{\dagger, \hat{\nu}}_1(s_1) - V^{\hat{\mu}, \dagger}_1(s_1) \le \epsilon$.
\end{proof}

%% file: proof_lower.tex

\section{Proofs of Hardness for Learning the Best Responses}
\label{appendix:proof-lower}
\label{appendix:proof-lower-bound-coupled}

In this section we give the proof of Theorem~\ref{theorem:lower-bound-best-response}, and Corollary~\ref{cor:lower-bound-coupled}. Our proof is inspired by a computational hardness result for adversarial MDPs in~\citep[Section 4.2]{yadkori2013online}, which constructs a family of adversarial MDPs that are computationally as hard as an agnostic parity learning problem.

Section \ref{app:hard_instance}, \ref{app:series_hard_pbs}, \ref{app:hard_together} will be devoted to prove Theorem \ref{theorem:lower-bound-best-response}, while Corollary \ref{cor:lower-bound-coupled} is proved in Section \ref{sec:adversarial_proof}.
Towards proving Theorem \ref{theorem:lower-bound-best-response}, we will:
\begin{itemize}
  \item (Section \ref{app:hard_instance}) Construct a Markov game. 
  \item (Section \ref{app:series_hard_pbs}) Define a series of problems where a solution in problem implies another.
  \item (Section \ref{app:hard_together}) Based on the believed computational hardness of learning paries with noise (Conjecture~\ref{conj:hardness}), we conclude that finding the best response of non-Markov policies is computationally hard.
\end{itemize}

\subsection{Markov game construction}
\label{app:hard_instance}
We now describe a Markov game inspired the adversarial MDP in \citep[Section 4.2]{yadkori2013online}.
We define a Markov game in which we have $2H$ states, $\left\{ i_0,i_1 \right\} _{i=2}^{H}$, $1_0$ (the initial state) and $\bot$ (the terminal state)\footnote{In \cite{yadkori2013online} the states are denoted by $\left\{ i_a,i_b \right\} _{i=2}^{H}$ instead. Here we slightly change the notation to make it different from the notation of the actions}. In each state the max-player has two actions $a_0$ and $a_1$, while the min-player has two actions $b_0$ and $b_1$. The transition kernel is deterministic and the next state for steps $h\le H-1$ is defined in Table~\ref{table:parity-learning-transition}:

\begin{table}[h]
  \centering
  \begin{tabular}{|c|c|c|c|c|} \hline
    State/Action & $(a_0,b_0)$ & $(a_0,b_1)$ & $(a_1,b_0)$ & $(a_1,b_1)$ \\ \hline
    $i_0$             & $(i+1)_0$   & $(i+1)_0$   & $(i+1)_0$   & $(i+1)_1$   \\ \hline
    $i_1$             & $(i+1)_1$   & $(i+1)_0$   & $(i+1)_1$   & $(i+1)_1$  \\\hline
  \end{tabular}
  \caption{Transition kernel of the hard instance.
  }
  \label{table:parity-learning-transition}
\end{table}
At the $H$-th step, i.e. states $H_0$ and $H_1$, the next state is always $\bot$ regardless of the action chosen by both players. The reward function is always $0$ except at the $H$-th step. The reward is determined by the action of the min-player, defined by 
\begin{table}[h]
  \centering
  \begin{tabular}{|c|c|c|} \hline
    State/Action & $(\cdot,b_0)$ & $(\cdot,b_1)$ \\ \hline
  $H_0$             & $1$           & $0$           \\ \hline
  $H_1$             & $0$           & $1$          \\ \hline
  \end{tabular}
  \caption{Reward of the hard instance.}
\end{table}

At the beginning of every episode $k$, both players pick their own policies $\mu_k$ and $\nu_k$, and execute them throughout the episode. The min-player can possibly pick her policy $\nu_k$ adaptive to all the observations in the earlier episodes. The only difference from the standard Markov game protocol is that the actions of the min-player except the last step will be revealed at the beginning of each episode, to match the setting in agnostic learning parities (Problem 2 below). Therefore we are actually considering a easier problem (for the max-player) and the lower bound naturally applies.

\subsection{A series of computationally hard problems}
\label{app:series_hard_pbs}
We first introduce a series of problems and then show how the reduction works.

\paragraph{Problem 1} The max-player $\epsilon$-approximates the best reponse for any general policy $\nu$ in the Markov game defined in Appendix~\ref{app:hard_instance} with probability at least $1/2$, in $\poly(H,1/\epsilon)$ time.

\paragraph{Problem 2} Let $x=\left( x_1,\cdots ,x_n \right) $ be a vector in $\left\{ \text{0,}1 \right\} ^n$, $T\subseteq \left[ n \right]$ and $0< \alpha < 1/2$.The parity of $x$ on $T$ is the boolean function $\phi _T\left( x \right) =\oplus _{i\in T}x_i$. In words, $\phi _T\left( x \right)$ outputs $0$ if the number of
ones in the subvector $(x_i)_{i \in T}$ is even and $1$ otherwise. A uniform query oracle for
this problem is a randomized algorithm that returns a random uniform vector $x$, as
well as a noisy classification $f(x)$ which is equal to $\phi_T(x)$ w.p. $\alpha$ and $1 -\phi_T(x)$ w.p. $1-\alpha$. All examples returned by the oracle are independent. The learning parity with noise problem consists in designing an algorithm with
access to the oracle such that,
\begin{itemize}
  \item (\textbf{Problem 2.1}) w.p at least $1/2$, find a (possibly random) function $h: \{0, 1\}^n \rightarrow \{0, 1\}$ satisy $\E_h P_x[h(x) \neq \phi_T(x)] \le \epsilon$, in $\poly(n ,1/\epsilon)$ time.
  \item (\textbf{Problem 2.2}) w.p at least $1/4$, find $h: \{0, 1\}^n \rightarrow \{0, 1\}$ satisy $P_x[h(x) \neq \phi_T(x)] \le \epsilon$, in $\poly(n ,1/\epsilon)$ time.
  \item  (\textbf{Problem 2.3}) w.p at least $1-p$, find $h: \{0, 1\}^n \rightarrow \{0, 1\}$ satisy $P_x[h(x) \neq \phi_T(x)] \le \epsilon$, in $\poly(n ,1/\epsilon, 1/p)$ time.
\end{itemize}

We remark that Problem 2.3 is the formal definition of learning parity with noise \citep[Definition 2]{mossel2005learning}, which is conjectured to be computationally hard in the community (see also Conjecture \ref{conj:hardness}).

\paragraph{Problem 2.3 reduces to Problem 2.2}

Step 1: Repeatly apply algorithm for Problem 2.2 $\ell$ times to get $h_1,\dots,h_\ell$ such that $\min_i P_x[h_i(x)\neq \phi_T(x)] \le \epsilon$ with probability at least $1-(3/4)^\ell$. This costs $\poly(n, \ell, 1/\epsilon)$ time. Let $i_\star = \argmin_i {\rm err}_i$ where ${\rm err}_i = P_x[h_i(x)\neq \phi_T(x)]$.

Step 2: Construct estimators using $N$ additional data ${(x^{(j)},y^{(j)})}_{j=1}^N$,
\begin{equation*}
  \hat{\rm err}_i \defeq \frac{\frac{1}{N} \sum_{j=1}^N \mathbb{I}\{ h_i(x^{(j)}) \neq y^{(j)} \} - \alpha}{1-2\alpha}.
\end{equation*}
Pick $\hat{i} = \argmin_i \hat{\rm err}_i$. When $N\ge \log(1/p)/\epsilon^2$, with probability at least $1-p/2$, we have
\begin{equation*}
  \max_i \left| \hat{\rm err}_i - {\rm err}_i \right| \le \frac{\epsilon}{1-2\alpha}.
\end{equation*}
This means that
\begin{equation*}
  {\rm err}_{\hat{i}} \le \hat{\rm err}_{\hat{i}} + \frac{\epsilon}{1-2\alpha} \le \hat{\rm err}_{i_\star} + \frac{\epsilon}{1-2\alpha} \le {\rm err}_{i_\star} + \frac{2\epsilon}{1-2\alpha} \le O(1)\epsilon.
\end{equation*}
This step uses $\poly(n, N, \ell)=\poly(n, 1/\epsilon, \log(1/p), \ell)$ time.

Step 3: Pick $\ell=\log(1/p)$, we are guaranteed that good events in step 1 and step 2 happen with probability $\ge 1-p/2$ and altogether happen with probability at least $1-p$. The total time used is $\poly(n, 1/\epsilon, \log(1/p))$. Note better dependence on $p$ than required.

\paragraph{Problem 2.2 reduces to Problem 2.1:}
If we have an algorithm that gives $\E_{h \sim \cD} P_x[h(x) \neq \phi_T(x)] \le \epsilon$ with probability $1/2$. Then if we sample $\hat{h} \sim \cD$, by Markov's inequality, we have with probability $\ge 1/4$ that 
\begin{equation*}
P_x[\hat{h}(x) \neq \phi_T(x)] \le 2\epsilon
\end{equation*}

\paragraph{Problem 2.1 reduces to Problem 1:}
Consider the Markov game constructed above with $H-1=n$. The only missing piece we fill up here is the policy $\nu$ of the min-player, which is constructed as following. The min-player draws a sample $(x,y)$ from the uniform query oracle, then taking action $b_0$ at the step $h \le H-1$ if $x_h=0$ and $b_1$ otherwise. For the $H$-th step, the min-player take action $b_0$ if $y=0$ and $b_1$ otherwise. Also notice the policy $\hat{\mu}$ of the max-player can be descibed by a set $\hat{T} \subseteq [H]$ where he takes action $a_1$ at step $h$ if $h $ and $a_0$ otherwise. As a result, the max-player receive non-zero result iff $\phi_{\hat{T}}(x)=y$.

In the Markov game, we have $V_1^{\hat{\mu}, \nu}(s_1) = \P(\phi_{\hat{T}}(x)=y)$. As a result, the optimal policy $\mu^{*}$ corresponds to the true parity set $T$. As a result,
$$
(V_1^{\dagger, \nu}-V_1^{\hat{\mu}, \nu})(s_1)=\P_{x, y}(\phi_T(x) = y) - \P_{x, y}(\phi_{\hat{T}}(x) = y) \le \epsilon
$$
by the $\epsilon$-approximation guarantee.

Also notice
\begin{align*}
 \P_{x, y}(\phi_{\hat{T}}(x) \neq y) - \P_{x, y}(\phi_{T}(x) \neq y)
 =& (1-\alpha) \P_x(\phi_{\hat{T}}(x) \neq \phi_{T}(x)) + \alpha \P_x(\phi_{\hat{T}}(x) = \phi_{T}(x)) - \alpha\\
 =& (1-2\alpha) \P_x(\phi_{\hat{T}}(x) \neq \phi_{T}(x))
\end{align*}

This implies:

\begin{equation*}
\P_x(\phi_{\hat{T}}(x) \neq \phi_{T}(x)) \le \frac{\epsilon}{1-2\alpha}
\end{equation*}

\subsection{Putting them together}
\label{app:hard_together}
So far, we have proved that Solving Problem 1 implies solving Problem 2.3, where Problem 1 is the problem of learning $\epsilon$-approximate best response in Markov games (the problem we are interested in), and Problem 2.3 is precisely the problem of learning parity with noise \cite{mossel2005learning}. This concludes the proof.

\subsection{Proofs of Hardness Against Adversarial Opponents}
\label{sec:adversarial_proof}
Corollary~\ref{cor:lower-bound-coupled} is a direct consequence of Theorem~\ref{theorem:lower-bound-best-response}, as we will show now.

\begin{proof}[Proof of Corollary~\ref{cor:lower-bound-coupled}]
We only need to prove a polynomial time no-regret algorithm also learns the best response in a Markov game where the min-player following non-Markov policy $\nu$. Then the no-regret guarantee implies,
$$
   V^{\dagger, \nu}_{1}
  (s_1) -\frac{1}{K}  \sum_{k=1}^K V^{\mu^k, \nu}_{1} (s_1) \le {\rm poly}(S, H, A, B)K^{-\delta}
$$
where $\mu_k$ is the policy of the max-player in the $k$-th episode. If we choose $\hat{\mu}$ uniformly randomly from $\set{\mu_k}_{k=1}^K$, then 
$$ V^{\dagger, \nu}_{1}
(s_1) - V^{\hat{\mu}, \nu}_{1}
(s_1)  \le {\rm poly}(S, H, A, B)K^{-\delta}.
$$

Choosing $\epsilon = {\rm poly}(S, H, A, B)K^{-\delta}$, $K={\rm poly}(S, H, A, B, 1/\epsilon)$ and the running time of the no-regret algorithm is still ${\rm poly}(S, H, A, B, 1/\epsilon)$ to learn the $\epsilon$-approximate best response.

To see that the Corollary~\ref{cor:lower-bound-coupled} remains to hold for policies that are Markovian in each episode and non-adaptive, we can take the hard instance in Theorem~\ref{theorem:lower-bound-best-response} and let $\nu^k$ denote the min-player's policy in the $k$-th episode. Note that each $\nu^k$ is Markovian and non-adaptive on the observations in previous episodes. If there is a polynomial time no-regret algorithm against such $\set{\nu^k}$, then by the online-to-batch conversion similar as the above, the mixture of $\set{\mu_k}_{k=1}^K$ learns a best response against $\nu$ in polynomial time.

\end{proof}

%% file: bandit.tex

\section{Auxiliary Lemmas for Weighted Adversarial Bandit}
\label{sec:bandit}

In this section, we formulate the bandit problem we reduced to in the proof of Lemma~\ref{lem:bandit-regret}. Although the machnisms are already well understood, we did not find a good reference of Follow the Regularized Leader (FTRL) algorithm with
\begin{enumerate}
   \item changing step size
   \item weighted regret
   \item high probability regret bound
\end{enumerate}

For completeness, we give the detailed derivation here.

\begin{algorithm}[h]
   \caption{FTRL for Weighted Regret with Changing Step Size}
   \label{algorithm:FTRL}
   \begin{algorithmic}[1]
      \FOR{episode $t=1,\dots,K$}
      \STATE $\theta_{t}(a) \propto \exp[-(\eta_{t}/w_t) \cdot \sum_{i=1}^{t-1}w_i \hat{l}_i(a)]$
      \STATE Take action $a_t \sim \theta_t(\cdot)$, and observe loss $\tilde{l}_t(a_t)$.
      \STATE $\hat{l}_t(a) \leftarrow \tilde{l}_t( a)\mathbb{I}\{a_t=a\}/(\theta_t (a) +\gamma_t)$ for all $a \in \cA$.
      
      \ENDFOR
   \end{algorithmic}
\end{algorithm}

We assume $\tilde{l}_i \in [0,1]^A$ and $\E_i \tilde{l}_i = l_i$. Define $A= |\cA|$, we set the hyperparameters by
\begin{equation*}
\eta_t= \gamma_t = \sqrt{\frac{\log A}{At}}
\end{equation*}

Define the filtration $\cF_t$ by the $\sigma$-algebra generated by $\{a_i,l_i\}_{i=1}^{t-1}$. Then the regret can be defined as 
$$
R_t\left( \theta^* \right) :=\sum_{i=1}^t{w_i\E_{a \sim \theta^*}[l_i(a)-l_i(a_i)|\cF_{i}]}
=\sum_{i=1}^t{w_i\left< \theta_i-\theta^*, l_i \right>}
$$

We can easily check the definitions here is just an abstract version of that in the proof of Lemma~\ref{lem:bandit-regret} with rescaling.  To state the regret guarantee, we also define $\iota=\log(p/AK)$ for any $p \in (0,1]$. Now we can upper bound the regret by
\begin{lemma}
    \label{lem:adv-bandit}
    Following Algorithm~\ref{algorithm:FTRL}, with probability $1-3p$, for any $ \theta^* \in \Delta ^A$ and $t \le K$ we have
    $$
    R_t\left( \theta^* \right) \le 2\max _{i\le t}w_i\sqrt{At\iota}+\frac{3\sqrt{A\iota}}{2}\sum_{i=1}^t{\frac{w_i}{\sqrt{i}}}+\frac{1}{2}\max _{i\le t}w_i\iota +\sqrt{2\iota \sum_{i=1}^t{w_{i}^{2}}}
    $$
\end{lemma}
\begin{proof}
    The regret $R_t(\theta^*)$ can be decomposed into three terms
\begin{align*}
   R_t\left( \theta^* \right) =&\sum_{i=1}^t{w_i\left< \theta_i-\theta^*,l_i \right>}
\\
=&\underset{\left( A \right)}{\underbrace{\sum_{i=1}^t{w_i\left< \theta_i-\theta^*,\hat{l}_i \right>}}}+\underset{\left( B \right)}{\underbrace{\sum_{i=1}^t{w_i\left< \theta_i,l_i-\hat{l}_i \right>}}}+\underset{\left( C \right)}{\underbrace{\sum_{i=1}^t{w_i\left< \theta^*,\hat{l}_i-l_i \right>}}}
\end{align*}
and we bound $(A)$ in Lemma~\ref{lem:bound-A}, $(B)$ in Lemma~\ref{lem:bound-B} and $(C)$ in Lemma~\ref{lem:bound-C}. 

Setting $\eta _t=\gamma _t=\sqrt{\frac{\log A}{At}}$, the conditions in Lemma~\ref{lem:bound-A} and Lemma~\ref{lem:bound-C} are satisfied. Putting them together and take union bound, we have with probability $1-3p$
 \begin{align*}
 R_t\left( \theta^* \right) \le& \frac{w_t\log A}{\eta _t}+\frac{A}{2}\sum_{i=1}^t{\eta _iw_{i}}+\frac{1}{2}\max _{i\le t}w_i\iota +A\sum_{i=1}^t{\gamma _iw_i}+\sqrt{2\iota \sum_{i=1}^t{w_{i}^{2}}}+\max _{i\le t}w_i\iota /\gamma _t
 \\
 \le& 2\max _{i\le t}w_i\sqrt{At\iota}+\frac{3\sqrt{A\iota}}{2}\sum_{i=1}^t{\frac{w_i}{\sqrt{i}}}+\frac{1}{2}\max _{i\le t}w_i\iota +\sqrt{2\iota \sum_{i=1}^t{w_{i}^{2}}}
 \end{align*}

\end{proof}

The rest of this section is devoted to the proofs of the Lemmas used in the proofs of Lemma~\ref{lem:adv-bandit}. We begin the following useful lemma adapted from Lemma 1 in \cite{neu2015explore}, which is crucial in constructing high probability guarantees.

\begin{lemma}
   \label{lem:Neu}
   For any sequence of coefficients $c_1, c_2, \ldots, c_t$ s.t. $c_i \in [0,2\gamma_i]^A$ is $\cF_i$-measurable, we have with probability $1-p/AK$,
$$
\sum_{i=1}^t{w_i\left< c_i,\hat{l}_i-l_i \right>}\le \max _{i\le t}w_i\iota 
$$
\end{lemma}

\begin{proof}
   Define $w=\max _{i\le t}w_i$. By definition, 
   \begin{align*}
      w_i\hat{l}_i\left( a \right) =&\frac{w_i\tilde{l}_i\left( a \right) \mathbb{I}\left\{ a_i=a \right\}}{\theta_i\left( a \right) +\gamma _i}
      \le \frac{w_i\tilde{l}_i\left( a \right) \mathbb{I}\left\{ a_i=a \right\}}{\theta_i\left( a \right) +\frac{w_i\tilde{l}_i\left( a \right) \mathbb{I}\left\{ a_i=a \right\}}{w}\gamma _i}
      \\
      =&\frac{w}{2\gamma _i}\frac{\frac{2\gamma _iw_i\tilde{l}_i\left( a \right) \mathbb{I}\left\{ a_i=a \right\}}{w \theta_i\left( a \right)}}{1+\frac{\gamma _iw_i\tilde{l}_i\left( a \right) \mathbb{I}\left\{ a_i=a \right\}}{w\theta_i\left( a \right)}}
      \overset{\left( i \right)}{\le} \frac{w}{2\gamma _i}\log \left( 1+\frac{2\gamma _iw_i\tilde{l}_i\left( a \right) \mathbb{I}\left\{ a_i=a \right\}}{w \theta_i\left( a \right)} \right)    
   \end{align*}
   where $(i)$ is because $\frac{z}{1+z/2}\le \log \left( 1+z \right) $ for all $z \ge 0$.

   Defining the sum
   $$
   \hat{S}_i=\frac{w_i}{w}\left< c_i,\hat{l}_i \right> , \,\,\, S_i=\frac{w_i}{w}\left< c_i,l_i \right>, 
   $$
we have
\begin{align*}
   \E_i\left[ \exp \left( \hat{S}_i \right) \right] &\le \E_i\left[ \exp \left( \sum_a{\frac{c_i\left( a \right)}{2\gamma _i}\log \left( 1+\frac{2\gamma _iw_i\tilde{l}_i\left( a \right) \mathbb{I}\left\{ a_i=a \right\}}{w \theta_i\left( a \right)} \right)} \right) \right] 
\\
&\overset{\left( i \right)}{\le}\E_i\left[ \prod_a{\left( 1+\frac{c_i\left( a \right) w_i\tilde{l}_i\left( a \right) \mathbb{I}\left\{ a_i=a \right\}}{w \theta_i\left( a \right)} \right)} \right] 
\\
&=\E_i\left[ 1+\sum_a{\frac{c_i\left( a \right) w_i\tilde{l}_i\left( a \right) \mathbb{I}\left\{ a_i=a \right\}}{w \theta_i\left( a \right)}} \right] 
\\
&=1+S_i
\le \exp \left( S_i \right) 
\end{align*}
where $(i)$ is because $z_1\log \left( 1+z_2 \right) \le \log \left( 1+z_1z_2 \right) $ for any $0 \le z_1 \ge 1$ and $z_2 \ge -1$. Here we are using the condition $c_i\left( a \right) \le 2\gamma _i$ to guarantee the condition is satisfied.

Equipped with the above bound, we can now prove the concentration result.
\begin{align*}
   \P\left[ \sum_{i=1}^t{\left( \hat{S}_i-S_i \right)}\ge \iota \right] &=\P\left[ \exp \left[ \sum_{i=1}^t{\left( \hat{S}_i-S_i \right)} \right] \ge \frac{AK}{p} \right] 
\\
&\le \frac{p}{AK}\E_t\left[ \exp \left[ \sum_{i=1}^t{\left( \hat{S}_i-S_i \right)} \right] \right] 
\\
&\le \frac{p}{AK}\E_{t-1}\left[ \exp \left[ \sum_{i=1}^{t-1}{\left( \hat{S}_i-S_i \right)} \right] E_t\left[ \exp \left( \hat{S}_t-S_t \right) \right] \right] 
\\
&\le \frac{p}{AK}\E_{t-1}\left[ \exp \left[ \sum_{i=1}^{t-1}{\left( \hat{S}_i-S_i \right)} \right] \right] 
\\
&\le \cdots \le \frac{p}{AK}
\end{align*}
The claim is proved by taking the union bound.
\end{proof}

Using Lemma~\ref{lem:Neu}, we can bound the $(A)(B)(C)$ separately as below.

\begin{lemma}
   \label{lem:bound-A}
   If $\eta _i \le 2 \gamma_i$  for all $i \le t$, with probability $1-p$, for any $t \in [K]$ and $\theta^* \in \Delta^A$, 
   $$
   \sum_{i=1}^t{w_i\left< \theta_i-\theta^*,\hat{l}_i \right>} \le \frac{w_t\log A}{\eta _t}+\frac{A}{2}\sum_{i=1}^t{\eta _iw_{i}}+\frac{1}{2}\max _{i\le t}w_i\iota 
   $$
\end{lemma}
\begin{proof}
   We use the standard analysis of FTRL with changing step size, see for example Exercise 28.13 in \cite{lattimore2018bandit}. Notice the essential step size is $\eta _t/w_t$,
   \begin{align*}
      \sum_{i=1}^t{w_i\left< \theta_i-\theta^*,\hat{l}_i \right>}&\le \frac{w_t\log A}{\eta _t}+\frac{1}{2}\sum_{i=1}^t{\eta _i w_i}\left< \theta_i,\hat{l}_{i}^{2} \right> 
\\
&\le \frac{w_t\log A}{\eta _t}+\frac{1}{2}\sum_{i=1}^t{\sum_{a \in \cA}{\eta _i}w_i\hat{l}_i\left( a \right)}
\\
&\overset{\left( i \right)}{\le}\frac{w_t\log A}{\eta _t}+\frac{1}{2}\sum_{i=1}^t{\sum_{a \in \cA}{\eta _i}w_il_i\left( a \right)}+\frac{1}{2}\max _{i\le t}w_i\iota  
\\
&\le \frac{w_t\log A}{\eta _t}+\frac{A}{2}\sum_{i=1}^t{\eta _i w_{i}}+\frac{1}{2}\max _{i\le t}w_i\iota 
   \end{align*}
where $(i)$ is by using Lemma~\ref{lem:Neu} with $c_i(a)=\eta _i$. The any-time guarantee is justifed by taking union bound.
\end{proof}

\begin{lemma}
   \label{lem:bound-B}
   With probability $1-p$, for any $t \in [K]$, 
$$
\sum_{i=1}^t{w_i\left< \theta_i,l_i-\hat{l}_i \right>} \le A\sum_{i=1}^t{\gamma _iw_i}+\sqrt{2\iota \sum_{i=1}^t{w_{i}^{2}}}
$$
\end{lemma}
\begin{proof}
 We further decopose it into 
 $$
 \sum_{i=1}^t{w_i\left< \theta_i,l_i-\hat{l}_i \right>}=\sum_{i=1}^t{w_i\left< \theta_i,l_i-\E_i\hat{l}_i \right>}+\sum_{i=1}^t{w_i\left< \theta_i,\E_i\hat{l}_i-\hat{l}_i \right>}
$$

The first term is bounded by
\begin{align*}
\sum_{i=1}^t{w_i\left< \theta_i,l_i-\E_i\hat{l}_i \right>}=&\sum_{i=1}^t{w_i\left< \theta_i,l_i-\frac{\theta_i}{\theta_i+\gamma _i}l_i \right>}
\\
=&\sum_{i=1}^t{w_i\left< \theta_i,\frac{\gamma _i}{\theta_i+\gamma _i}l_i \right>}
\le A\sum_{i=1}^t{\gamma _iw_i}
\end{align*}

To bound the second term, notice
$$
\left< \theta_i,\hat{l}_i \right> \le \sum_{a \in \cA}{\theta_i\left( a \right)}\frac{\mathbb{I}\left\{ a_t=a \right\}}{\theta_i(a)+\gamma _i}\le \sum_{a \in \cA}{\mathbb{I}\left\{ a_i=a \right\}}=1,
$$
thus $\{w_i\left< \theta_i,\E_i\hat{l}_i-\hat{l}_i \right>\}_{i=1}^t$ is a bounded martingale difference sequence w.r.t. the filtration $\{\cF_i\}_{i=1}^t$. By Azuma-Hoeffding,
$$
\sum_{i=1}^t{\left< \theta_i,\E_i\hat{l}_i-\hat{l}_i \right>}\le \sqrt{2\iota \sum_{i=1}^t{w_{i}^{2}}}
$$
\end{proof}

\begin{lemma}
   \label{lem:bound-C}
   With probability $1-p$, for any $t \in [K]$ and any $\theta^* \in \Delta^A$, if $\gamma_i$ is non-increasing in $i$,
   $$
   \sum_{i=1}^t{w_i\left< \theta^*,\hat{l}_i-l_i \right>}\le \max _{i\le t}w_i\iota /\gamma _t
   $$  
\end{lemma}
\begin{proof}
   Define a basis $\left\{ e_j \right\} _{j=1}^{A}$ of $\mathbb{R}^A$ by 
   $$
   e_j\left( a \right) =\begin{cases}
      \text{1 if }a=j\\
      \text{0 otherwise}\\
   \end{cases}
   $$

   Then for all the $j \in [A]$, apply Lemma~\ref{lem:Neu} with $c_i=\gamma_t e_j$. Sine now $c_i(a)\le \gamma_t \le \gamma_i$, the condition in Lemma~\ref{lem:Neu} is satisfied. As a result, 
   $$
   \sum_{i=1}^t{w_i\left< e_j,\hat{l}_i-l_i \right>}\le \max _{i\le t}w_i\iota /\gamma _t
   $$ 
   
   Since any $\theta^*$ is a convex combination of $\left\{ e_j \right\} _{j=1}^{A}$, by taking the union bound over $j \in [A]$, we have
   $$
   \sum_{i=1}^t{w_i\left< \theta^*,\hat{l}_i-l_i \right>}\le \max _{i\le t}w_i\iota /\gamma _t
   $$
\end{proof}

%% file: main.bbl
\begin{thebibliography}{38}
\providecommand{\natexlab}[1]{#1}
\providecommand{\url}[1]{\texttt{#1}}
\expandafter\ifx\csname urlstyle\endcsname\relax
  \providecommand{\doi}[1]{doi: #1}\else
  \providecommand{\doi}{doi: \begingroup \urlstyle{rm}\Url}\fi

\bibitem[Azar et~al.(2017)Azar, Osband, and Munos]{azar2017minimax}
Mohammad~Gheshlaghi Azar, Ian Osband, and R{\'e}mi Munos.
\newblock Minimax regret bounds for reinforcement learning.
\newblock In \emph{Proceedings of the 34th International Conference on Machine
  Learning-Volume 70}, pages 263--272. JMLR. org, 2017.

\bibitem[Bai and Jin(2020)]{bai2020provable}
Yu~Bai and Chi Jin.
\newblock Provable self-play algorithms for competitive reinforcement learning.
\newblock \emph{arXiv preprint arXiv:2002.04017}, 2020.

\bibitem[Baker et~al.(2020)Baker, Kanitscheider, Markov, Wu, Powell, McGrew,
  and Mordatch]{baker2020emergent}
Bowen Baker, Ingmar Kanitscheider, Todor Markov, Yi~Wu, Glenn Powell, Bob
  McGrew, and Igor Mordatch.
\newblock Emergent tool use from multi-agent autocurricula.
\newblock In \emph{International Conference on Learning Representations}, 2020.
\newblock URL \url{https://openreview.net/forum?id=SkxpxJBKwS}.

\bibitem[Brafman and Tennenholtz(2002)]{brafman2002r}
Ronen~I Brafman and Moshe Tennenholtz.
\newblock R-max-a general polynomial time algorithm for near-optimal
  reinforcement learning.
\newblock \emph{Journal of Machine Learning Research}, 3\penalty0
  (Oct):\penalty0 213--231, 2002.

\bibitem[Brambilla et~al.(2013)Brambilla, Ferrante, Birattari, and
  Dorigo]{brambilla2013swarm}
Manuele Brambilla, Eliseo Ferrante, Mauro Birattari, and Marco Dorigo.
\newblock Swarm robotics: a review from the swarm engineering perspective.
\newblock \emph{Swarm Intelligence}, 7\penalty0 (1):\penalty0 1--41, 2013.

\bibitem[Brown and Sandholm(2019)]{brown2019superhuman}
Noam Brown and Tuomas Sandholm.
\newblock Superhuman ai for multiplayer poker.
\newblock \emph{Science}, 365\penalty0 (6456):\penalty0 885--890, 2019.

\bibitem[Dann et~al.(2017)Dann, Lattimore, and Brunskill]{dann2017unifying}
Christoph Dann, Tor Lattimore, and Emma Brunskill.
\newblock Unifying pac and regret: Uniform pac bounds for episodic
  reinforcement learning.
\newblock In \emph{Advances in Neural Information Processing Systems}, pages
  5713--5723, 2017.

\bibitem[Dann et~al.(2019)Dann, Li, Wei, and Brunskill]{dann2019policy}
Christoph Dann, Lihong Li, Wei Wei, and Emma Brunskill.
\newblock Policy certificates: Towards accountable reinforcement learning.
\newblock In \emph{International Conference on Machine Learning}, pages
  1507--1516, 2019.

\bibitem[Filar and Vrieze(2012)]{filar2012competitive}
Jerzy Filar and Koos Vrieze.
\newblock \emph{Competitive Markov decision processes}.
\newblock Springer Science \& Business Media, 2012.

\bibitem[Hansen et~al.(2013)Hansen, Miltersen, and Zwick]{hansen2013strategy}
Thomas~Dueholm Hansen, Peter~Bro Miltersen, and Uri Zwick.
\newblock Strategy iteration is strongly polynomial for 2-player turn-based
  stochastic games with a constant discount factor.
\newblock \emph{Journal of the ACM (JACM)}, 60\penalty0 (1):\penalty0 1--16,
  2013.

\bibitem[Hu and Wellman(2003)]{hu2003nash}
Junling Hu and Michael~P Wellman.
\newblock Nash q-learning for general-sum stochastic games.
\newblock \emph{Journal of machine learning research}, 4\penalty0
  (Nov):\penalty0 1039--1069, 2003.

\bibitem[Jaksch et~al.(2010)Jaksch, Ortner, and Auer]{jaksch2010near}
Thomas Jaksch, Ronald Ortner, and Peter Auer.
\newblock Near-optimal regret bounds for reinforcement learning.
\newblock \emph{Journal of Machine Learning Research}, 11\penalty0
  (Apr):\penalty0 1563--1600, 2010.

\bibitem[Jia et~al.(2019)Jia, Yang, and Wang]{jia2019feature}
Zeyu Jia, Lin~F Yang, and Mengdi Wang.
\newblock Feature-based q-learning for two-player stochastic games.
\newblock \emph{arXiv preprint arXiv:1906.00423}, 2019.

\bibitem[Jin et~al.(2018)Jin, Allen-Zhu, Bubeck, and Jordan]{jin2018q}
Chi Jin, Zeyuan Allen-Zhu, Sebastien Bubeck, and Michael~I Jordan.
\newblock Is {Q}-learning provably efficient?
\newblock In \emph{Advances in Neural Information Processing Systems}, pages
  4868--4878, 2018.

\bibitem[Jin et~al.(2019)Jin, Jin, Luo, Sra, and Yu]{jin2019learning}
Chi Jin, Tiancheng Jin, Haipeng Luo, Suvrit Sra, and Tiancheng Yu.
\newblock Learning adversarial markov decision processes with bandit feedback
  and unknown transition.
\newblock \emph{arXiv preprint arXiv:1912.01192}, 2019.

\bibitem[Kearns(1998)]{kearns1998efficient}
Michael Kearns.
\newblock Efficient noise-tolerant learning from statistical queries.
\newblock \emph{Journal of the ACM (JACM)}, 45\penalty0 (6):\penalty0
  983--1006, 1998.

\bibitem[Lattimore and Szepesv{\'a}ri(2018)]{lattimore2018bandit}
Tor Lattimore and Csaba Szepesv{\'a}ri.
\newblock Bandit algorithms.
\newblock 2018.

\bibitem[Littman(1994)]{littman1994markov}
Michael~L Littman.
\newblock Markov games as a framework for multi-agent reinforcement learning.
\newblock In \emph{Machine learning proceedings 1994}, pages 157--163.
  Elsevier, 1994.

\bibitem[Littman(2001)]{littman2001friend}
Michael~L Littman.
\newblock Friend-or-foe q-learning in general-sum games.
\newblock In \emph{ICML}, volume~1, pages 322--328, 2001.

\bibitem[Mossel and Roch(2005)]{mossel2005learning}
Elchanan Mossel and S{\'e}bastien Roch.
\newblock Learning nonsingular phylogenies and hidden markov models.
\newblock In \emph{Proceedings of the thirty-seventh annual ACM symposium on
  Theory of computing}, pages 366--375, 2005.

\bibitem[Neu(2015)]{neu2015explore}
Gergely Neu.
\newblock Explore no more: Improved high-probability regret bounds for
  non-stochastic bandits.
\newblock In \emph{Advances in Neural Information Processing Systems}, pages
  3168--3176, 2015.

\bibitem[OpenAI(2018)]{openaidota}
OpenAI.
\newblock Openai five.
\newblock \url{https://blog.openai.com/openai-five/}, 2018.

\bibitem[Osband and Van~Roy(2016)]{osband2016lower}
Ian Osband and Benjamin Van~Roy.
\newblock On lower bounds for regret in reinforcement learning.
\newblock \emph{arXiv preprint arXiv:1608.02732}, 2016.

\bibitem[Osband et~al.(2014)Osband, Van~Roy, and Wen]{osband2014generalization}
Ian Osband, Benjamin Van~Roy, and Zheng Wen.
\newblock Generalization and exploration via randomized value functions.
\newblock \emph{arXiv preprint arXiv:1402.0635}, 2014.

\bibitem[Radanovic et~al.(2019)Radanovic, Devidze, Parkes, and
  Singla]{radanovic2019learning}
Goran Radanovic, Rati Devidze, David Parkes, and Adish Singla.
\newblock Learning to collaborate in markov decision processes.
\newblock In \emph{International Conference on Machine Learning}, pages
  5261--5270, 2019.

\bibitem[Rosenberg and Mansour(2019)]{rosenberg2019online}
Aviv Rosenberg and Yishay Mansour.
\newblock Online convex optimization in adversarial markov decision processes.
\newblock \emph{arXiv preprint arXiv:1905.07773}, 2019.

\bibitem[Shalev-Shwartz et~al.(2016)Shalev-Shwartz, Shammah, and
  Shashua]{shalev2016safe}
Shai Shalev-Shwartz, Shaked Shammah, and Amnon Shashua.
\newblock Safe, multi-agent, reinforcement learning for autonomous driving.
\newblock \emph{arXiv preprint arXiv:1610.03295}, 2016.

\bibitem[Shapley(1953)]{shapley1953stochastic}
Lloyd~S Shapley.
\newblock Stochastic games.
\newblock \emph{Proceedings of the national academy of sciences}, 39\penalty0
  (10):\penalty0 1095--1100, 1953.

\bibitem[Sidford et~al.(2019)Sidford, Wang, Yang, and Ye]{sidford2019solving}
Aaron Sidford, Mengdi Wang, Lin~F Yang, and Yinyu Ye.
\newblock Solving discounted stochastic two-player games with near-optimal time
  and sample complexity.
\newblock \emph{arXiv preprint arXiv:1908.11071}, 2019.

\bibitem[Silver et~al.(2016)Silver, Huang, Maddison, Guez, Sifre, Van
  Den~Driessche, Schrittwieser, Antonoglou, Panneershelvam, Lanctot,
  et~al.]{silver2016mastering}
David Silver, Aja Huang, Chris~J Maddison, Arthur Guez, Laurent Sifre, George
  Van Den~Driessche, Julian Schrittwieser, Ioannis Antonoglou, Veda
  Panneershelvam, Marc Lanctot, et~al.
\newblock Mastering the game of go with deep neural networks and tree search.
\newblock \emph{nature}, 529\penalty0 (7587):\penalty0 484, 2016.

\bibitem[Silver et~al.(2017)Silver, Schrittwieser, Simonyan, Antonoglou, Huang,
  Guez, Hubert, Baker, Lai, Bolton, et~al.]{silver2017mastering}
David Silver, Julian Schrittwieser, Karen Simonyan, Ioannis Antonoglou, Aja
  Huang, Arthur Guez, Thomas Hubert, Lucas Baker, Matthew Lai, Adrian Bolton,
  et~al.
\newblock Mastering the game of go without human knowledge.
\newblock \emph{nature}, 550\penalty0 (7676):\penalty0 354--359, 2017.

\bibitem[Strehl et~al.(2006)Strehl, Li, Wiewiora, Langford, and
  Littman]{strehl2006pac}
Alexander~L Strehl, Lihong Li, Eric Wiewiora, John Langford, and Michael~L
  Littman.
\newblock {PAC} model-free reinforcement learning.
\newblock In \emph{International Conference on Machine Learning}, pages
  881--888, 2006.

\bibitem[Vinyals et~al.(2019)Vinyals, Babuschkin, Czarnecki, Mathieu, Dudzik,
  Chung, Choi, Powell, Ewalds, Georgiev, et~al.]{vinyals2019grandmaster}
Oriol Vinyals, Igor Babuschkin, Wojciech~M Czarnecki, Micha{\"e}l Mathieu,
  Andrew Dudzik, Junyoung Chung, David~H Choi, Richard Powell, Timo Ewalds,
  Petko Georgiev, et~al.
\newblock Grandmaster level in starcraft ii using multi-agent reinforcement
  learning.
\newblock \emph{Nature}, 575\penalty0 (7782):\penalty0 350--354, 2019.

\bibitem[Watkins(1989)]{watkins1989learning}
Christopher John Cornish~Hellaby Watkins.
\newblock Learning from delayed rewards.
\newblock 1989.

\bibitem[Wei et~al.(2017)Wei, Hong, and Lu]{wei2017online}
Chen-Yu Wei, Yi-Te Hong, and Chi-Jen Lu.
\newblock Online reinforcement learning in stochastic games.
\newblock In \emph{Advances in Neural Information Processing Systems}, pages
  4987--4997, 2017.

\bibitem[Xie et~al.(2020)Xie, Chen, Wang, and Yang]{xie2020learning}
Qiaomin Xie, Yudong Chen, Zhaoran Wang, and Zhuoran Yang.
\newblock Learning zero-sum simultaneous-move markov games using function
  approximation and correlated equilibrium.
\newblock \emph{arXiv preprint arXiv:2002.07066}, 2020.

\bibitem[Yadkori et~al.(2013)Yadkori, Bartlett, Kanade, Seldin, and
  Szepesv{\'a}ri]{yadkori2013online}
Yasin~Abbasi Yadkori, Peter~L Bartlett, Varun Kanade, Yevgeny Seldin, and Csaba
  Szepesv{\'a}ri.
\newblock Online learning in markov decision processes with adversarially
  chosen transition probability distributions.
\newblock In \emph{Advances in neural information processing systems}, pages
  2508--2516, 2013.

\bibitem[Zimin and Neu(2013)]{zimin2013online}
Alexander Zimin and Gergely Neu.
\newblock Online learning in episodic markovian decision processes by relative
  entropy policy search.
\newblock In \emph{Advances in neural information processing systems}, pages
  1583--1591, 2013.

\end{thebibliography}
